\documentclass[journal]{IEEEtran}
\UseRawInputEncoding
\IEEEoverridecommandlockouts                              

\usepackage{amsmath,amsfonts, amssymb}
\usepackage[linesnumbered,ruled,vlined]{algorithm2e}
\usepackage{array}
\usepackage{subfigure}
\usepackage{textcomp}
\usepackage{stfloats}
\usepackage{url}
\usepackage{verbatim}
\usepackage{graphicx}
\usepackage{cite}
\usepackage{mathrsfs}
\usepackage{graphicx}
\usepackage{epstopdf}
\usepackage{tabularx} 
\usepackage{multirow}
\usepackage[ruled, vlined, linesnumbered]{algorithm2e}
\usepackage{color}
\usepackage{subcaption} 
\usepackage{booktabs}
\usepackage{hyperref}
\usepackage{makecell}
\newtheorem{theorem}{Theorem}

\newtheorem{proof}{Proof}[section]
\newtheorem{definition}{Definition}

\title{\LARGE \bf
	Multi-Agent Pathfinding with Non-Unit Integer Edge Costs via Enhanced Conflict-Based Search and Graph Discretization
}
\author{Hongkai Fan, Qinjing Xie, Bo Ouyang, Yaonan Wang, Zhi Yan, Jiawen He and Zheng Fang 
	\thanks{Hongkai Fan, Qinjing Xie, Bo Ouyang, Yaonan Wang, Zhi Yan and Jiawen He are affiliated with the College of Electrical and Information Engineering at Hunan University. Zheng Fang is associated with China Mobile Group Hunan Company Limited, Changsha 410082, China (email: fanhk, xieqinjing, ouyangbo, yaonan, yanzhi, hjwdaxia@hnu.edu.cn, fangzheng@hn.chinamobile.com).}}

\begin{document}
	\maketitle
	\thispagestyle{empty}
	\pagestyle{empty}

	\begin{abstract}
		
		Multi-Agent Pathfinding (MAPF) plays a critical role in various domains. Traditional MAPF methods typically assume unit edge costs and single-timestep actions, which limit their applicability to real-world scenarios. 
		To address this, MAPF$_R$ extends the MAPF to handle non-unit costs, thereby introducing more realistic modeling through real-valued edge costs and continuous-time actions. However, since most MAPF$_R$ formulations adopt a geometric collision model that allows conflicts to occur at any point along an edge, the resulting state space becomes unbounded, which significantly compromises the efficiency of existing solvers.
		In this paper, we propose MAPF$_Z$, a novel MAPF variant on graphs with non-unit integer costs. MAPF$_Z$ supports agents operating on flexible graphs without involving continuous time, thus preserving a finite state space while offering improved realism over classical MAPF.
		To efficiently solve MAPF$_Z$, we develop CBS-NIC, an enhanced Conflict-Based Search framework that incorporates time-interval-based conflict detection and constraints addition at the high level and an improved Safe Interval Path Planning (SIPP) algorithm at the low level. Additionally, we propose Bayesian Optimization for Graph Design (BOGD), a discretization method for non-unit edge costs that balances efficiency and accuracy, along with a sub-linear regret bound that ensures performance improves over time.
		Extensive experiments demonstrate that our approach outperforms state-of-the-art methods in runtime and success rate across diverse benchmark scenarios.
	
	\end{abstract}
	

	\section{INTRODUCTION}
	
	Multi-agent systems are widely applied in various domains such as smart factories, drone swarm coordination, and operations at ports and terminals \cite{c3}. A key challenge in these systems is multi-agent pathfinding (MAPF), an NP-hard problem \cite{c4} that focuses on planning collision-free paths for multiple agents to reach their respective goals from given start positions in a shared environment, while minimizing the makespan.
	In a classical MAPF problem, the goal is to coordinate a set of agents operating on a graph with unit edge costs (see Fig.~\ref{fig1a}).
	However, many existing MAPF solvers rely on overly simplified and sometimes unrealistic assumptions that limit their practical applicability. 
	For instance, these solvers often assume that all edges in the graph have the same length and that all agents move synchronously, with each move action taking one time step. This assumption implies either that all agents must travel at the same speed or that they must adjust their velocity and acceleration profiles so that each move fits precisely into a single time step, ensuring all agents start and finish their movements simultaneously. 
	
	\begin{figure}[t]
		\centering
		\subfigure[]{
			\includegraphics[width=0.29\linewidth]{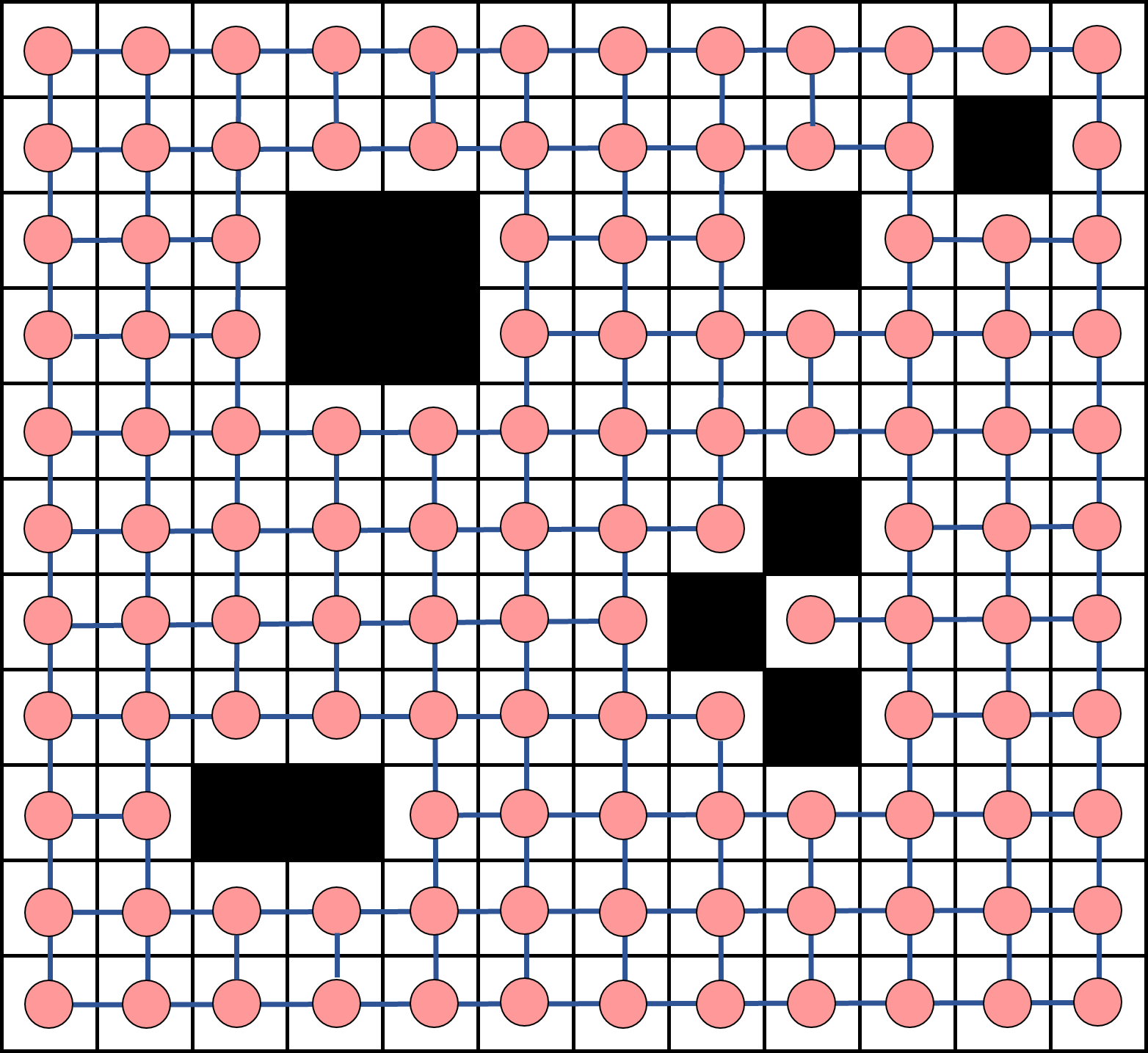}
			\label{fig1a}
		}
		\subfigure[]{
			\includegraphics[width=0.29\linewidth]{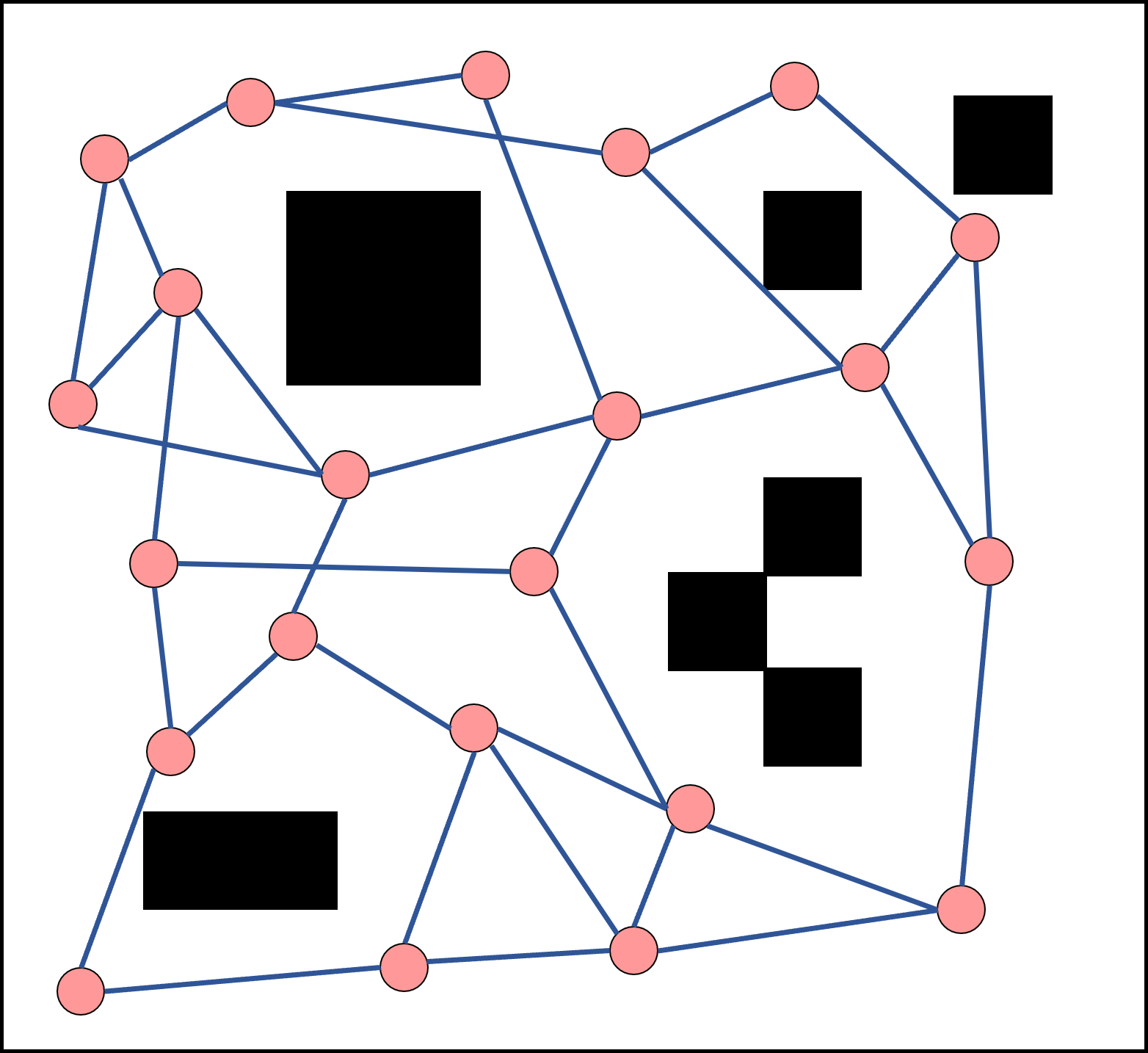}
			\label{fig1b}
		}
		\subfigure[]{
			\includegraphics[width=0.29\linewidth]{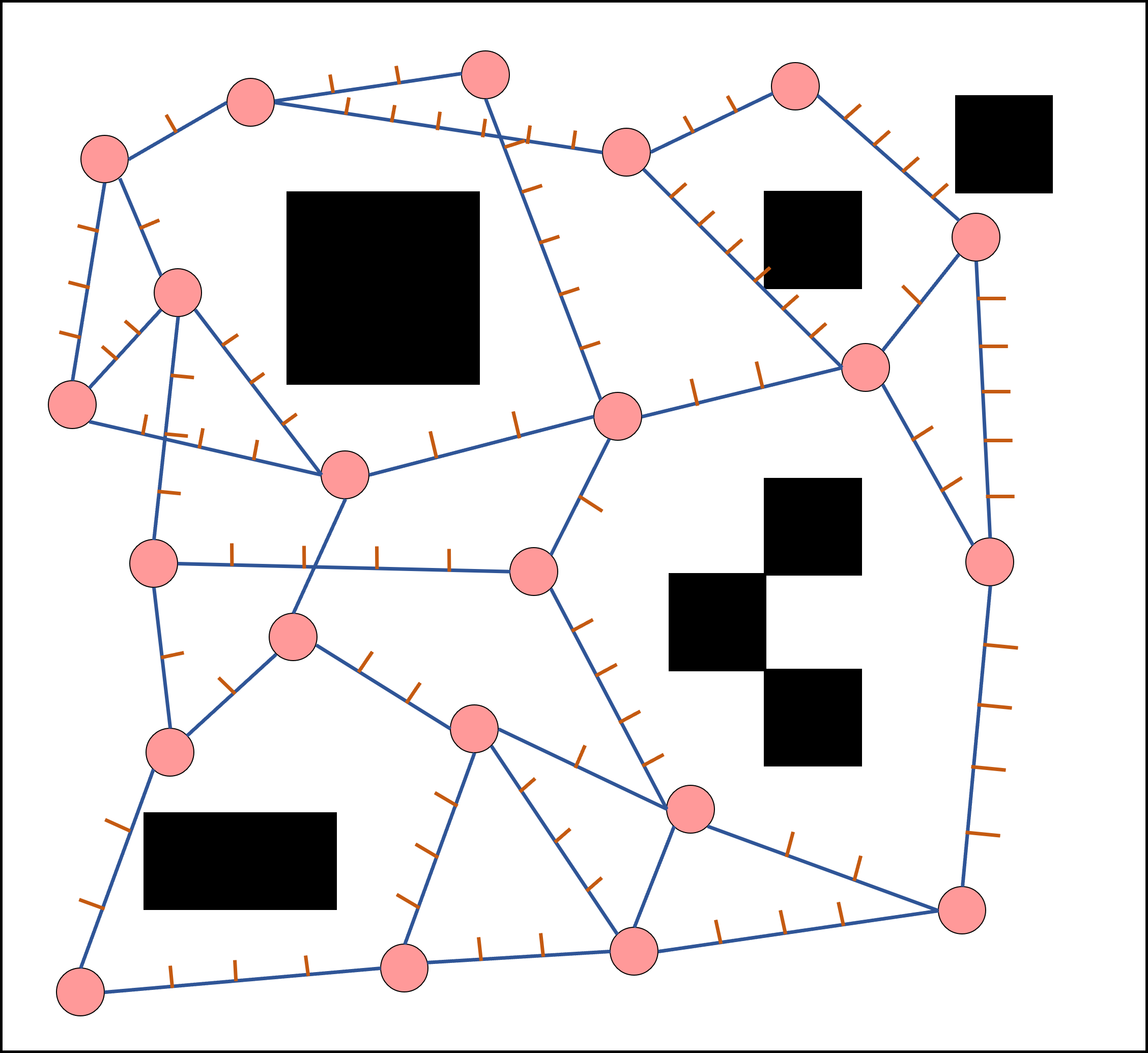}
			\label{fig1c}
		}
		\caption{Examples of graphs with different edge costs. (a) Unit cost graph. (b) Non-unit cost graph. (c) Non-unit integer cost graph.}
		\label{fig1}
	\end{figure}
	
	To overcome the unrealistic assumptions in the traditional MAPF formulation, several variants have been introduced, among which MAPF$_R$ \cite{c5} is a representative extension. This variant enhances the realism of MAPF by incorporating features such as asynchronous agent actions and weighted cost graphs. In MAPF$_R$, each vertex \( v \in V \) is assigned a positive traversal cost \( w(v) \in \mathbb{R}_{>0} \) (see Fig.~\ref{fig1b}), allowing the system to model agents operating in continuous metric spaces with heterogeneous movement costs.
	While these extensions significantly improve the applicability of MAPF in real-world scenarios, they also introduce additional computational challenges. In particular, many MAPF$_R$ formulations adopt a geometric collision model, where agents may collide at any point along edges in continuous time. This type of collision modeling leads to an unbounded state space, which severely hampers search efficiency and degrades solver performance.
	For instance, Continuous-time Conflict-Based Search (CCBS) \cite{c6} extends the classical CBS algorithm to support continuous time and variable-duration actions. Although CCBS guarantees optimality, its scalability is limited, typically handling only about 20 agents within a few minutes of runtime.
	To address asynchronous actions without relying heavily on geometric collision checking, the Loosely Synchronized Search (LSS) method \cite{ren2021loosely} adopts a duration-based conflict model. While this formulation ensures a bounded state space, scalability remains constrained; as reported in \cite{ren2021loosely}, the success rate decreases sharply with increasing agent numbers and approaches zero around 20 agents on most benchmark maps.
	To further enhance scalability, the Loosely Synchronized Rule-based Planning (LSRP) algorithm \cite{c7} introduces a caching mechanism that loosely synchronizes temporally close actions with different start times. Experiments show that LSRP can scale to an order of magnitude more agents than CCBS with substantially reduced runtimes. However, this gain comes at the cost of an average makespan increase of about 25\%. Its extension, LSRP* \cite{zhou2025lsrp}, improves solution quality through iterative anytime refinement, partially mitigating this drawback.
	Overall, these methods still struggle to balance scalability and solution quality, and tend to perform better in relatively simple scenarios.
	
	In this paper, we address the limitations of single-timestep actions and unit costs in the classical MAPF problem, as well as the low efficiency in solving the MAPF$_R$ variant. To this end, we introduce a new problem formulation, MAPF$_Z$, which extends MAPF to non-unit integer cost graphs (see Fig.~\ref{fig1c}). This formulation provides a better balance between solution efficiency and adaptability to real-world environments. 
	To address the challenges posed by non-unit integer costs in MAPF$_Z$, we propose CBS-NIC, a novel variant of CBS algorithm tailored for this setting. The key difficulty lies in transforming the original non-unit cost graphs into a suitable form with discrete integer edge costs. To this end, we introduce a discretization strategy guided by our proposed Bayesian Optimization for Graph Design (BOGD) method.
	BOGD is formulated as a bi-level optimization framework: the upper-level problem adopts Multi-Objective Bayesian Optimization (MOBO) to explore the space of candidate graph structures, while the lower-level problem evaluates each candidate by solving the MAPF$_Z$ instance using CBS-NIC. This design enables the generation of discretized graphs that strike a balance between fidelity to the original cost structure and compatibility with efficient pathfinding.

	The main contributions of this work can be itemized as follows:
	\begin{itemize}
		\item We introduce MAPF$_Z$, a novel MAPF variant that supports non-unit integer edge costs. It effectively bridges the gap between the idealized assumptions of traditional MAPF and the high computational complexity of weighted-cost graphs, offering a practical trade-off between modeling realism and efficiency.
		
		\item We propose CBS-NIC, an enhanced CBS variant designed for MAPF$_Z$ with non-unit integer edge costs. CBS-NIC introduces time-interval-based conflict detection, constraint definition, and an upgraded SIPP algorithm. It also incorporates improved heuristics, including enhanced versions of prioritizing conflicts (PC) and disjoint splitting (DS), to enhance planning efficiency.
		
		\item To discretize edge lengths in non-unit cost graphs, we propose BOGD, which balances solution efficiency and discretization error. Furthermore, we establish a sub-linear regret bound with respect to the number of iterations, demonstrating that the method becomes increasingly efficient and effective over time in minimizing cumulative regret.
		
	\end{itemize}
	
	The remainder of the paper is organized as follows: Section~\ref{sec2} reviews related work on MAPF. Section~\ref{sec3} formally defines the MAPF$_Z$ problem and revisits CBS along with relevant heuristic techniques. Section~\ref{sec4} details the implementation of CBS-NIC and the overall framework of BOGD. Section~\ref{sec5} presents the experimental results, and Section~\ref{sec6} concludes the paper with a summary and future directions.
	
	\section{RELATED WORK}\label{sec2}
	Path planning problem has been a popular study in the field of robotics for decades\cite{c10}. 
	In this section, we review the most relevant work from three complementary perspectives, covering MAPF with unit costs, MAPF with non-unit costs and continuous time, as well as graph design methods.
	
	For MAPF with unit costs, solvers can be categorized into optimal, bounded sub-optimal, and unbounded sub-optimal solvers. 
	Optimal solvers guarantee sum of cost (SOC) optimality and the most commonly used optimal MAPF solver is CBS\cite{c11}. CBS operates with a binary tree search at the high level, based on path conflicts, and uses A* at the low level to plan paths under conflict constraints. Variants of CBS include CBSH \cite{c12}, and CBSH2 \cite{c13}, which enhance heuristics using conflict and dependency graphs. \cite{c14} use constraints to reduce conflict tree nodes.
	Enhanced Partial Expansion A* (EPEA*)\cite{c22} is a heuristic-driven A* variant that selectively expands only minimal nodes using domain-specific knowledge, achieving improved efficiency in time-sensitive search across diverse benchmark domains.
	Although optimal solvers ensure the minimum makespan, they often require significant computation time, especially in complex environments and with many agents. 
	Bounded sub-optimal solvers provide solutions within a factor of \((1 + \epsilon)\) of the optimal cost, balancing solution quality and speed. Enhanced CBS (ECBS) \cite{c16} accelerates CBS by using focal search at both levels, while Explicit Estimation CBS (EECBS) \cite{c17} enhances this with explicit estimation search at the high level. The quality of solutions depends on \(\epsilon\), with larger values resulting in lower quality.
	Unbounded sub-optimal solvers prioritize speed, though they may compromise solution quality. For instance, 
	Priority Inheritance with Backtracking (PIBT) \cite{c18} plans in one-timestep increments, assigning unique priorities to agents at each step. These solvers generally focus on speed over solution quality.
	
	For MAPF with non-unit costs, Walker et al.~\cite{c19} introduced the MAPF$_R$ formulation and extended the Increasing cost tree search (ICTS) algorithm by constructing modified MDDs with cost intervals and reformulating the high-level search accordingly.
	Improved CBS with Helpful Bypass (ICBS-HB) \cite{c5} enhances the standard CBS framework by introducing a conflict bypassing strategy during path planning.
	ECBS with continuous time (ECBS-CT) \cite{c20} adapted the CBS framework to continuous time, and developed a bounded sub-optimal extension of SIPP to manage discrete time steps and rectilinear movements between neighboring vertices. Walker et al. \cite{c21} introduced CBS with constraint layering (CBS-CL), an extension of CBS that deals with non-unit edge costs and hierarchical movement abstractions. Although CBS-CL is complete in its solutions, it does not account for continuous time and does not ensure optimality. CCBS \cite{c6} builds on CBS by integrating an adapted version of SIPP for low-level search and introduces new conflict types and constraints at the high level, effectively addressing continuous time and geometric constraints. 
	
	
	Graph design plays a critical role in MAPF, as it directly affects the efficiency of planning algorithms. One of the most widely used approaches is the probabilistic roadmap (PRM) \cite{geraerts2004comparative}, which constructs a graph by sampling collision-free configurations in the continuous space and connecting nearby nodes via a local planner to form feasible edges. The resulting roadmap captures the connectivity of the environment and enables efficient path search between start and goal configurations. 
	Building upon PRM, a large body of work has explored more advanced roadmap construction methods for multi-agent settings. For example, Cooperative Timed Roadmaps (CTRMs) \cite{okumura2022ctrms} incorporate both spatial and temporal information to better coordinate agents and reduce conflicts. MRPT \cite{neto2018multi} extends PRM to distributed multi-robot systems, enabling decentralized coordination with reduced communication overhead. Other approaches focus on optimizing roadmap structures, such as ODRM \cite{henkel2020optimized}, which learns directed graphs with collision-avoidance patterns. Additionally, topology aware roadmap generation methods \cite{stenzel2022automated} and Voronoi-based heuristics \cite{wang2015voronoi} exploit geometric properties of the environment to improve planning efficiency.
	Despite these advances, most existing graph design methods are not jointly optimized with the underlying MAPF problem. They either rely on continuous representations with high computational overhead or construct graph structures without explicitly accounting for the interplay between temporal discretization and conflict resolution. As a result, the resulting representations are not well aligned with the requirements of efficient multi-agent planning, leading to limited scalability in dense, large-scale scenarios.

	\section{Problem definition}\label{sec3}
	In this section, we first present MAPF$_Z$, a new variant tailored for scenarios with non-unit costs, and then review the CBS algorithm along with relevant heuristic techniques.
	
	\subsection{Definition of MAPF$_Z$ with non-unit integer costs}
	
	The MAPF$_Z$ problem is a variant of the classical MAPF problem, introduced to model more realistic temporal behaviors by incorporating non-unit, positive integer edge costs. Formally, MAPF$_Z$ is defined on a positively weighted graph $\mathcal{G} = (V, E, W)$, where $V$ is the set of vertices, $E$ is the set of edges, and $W: E \rightarrow \mathbb{Z}_{>0}$ assigns a positive integer cost to each edge $e \in E$. These edge costs represent the durations required for agents to traverse the corresponding edges, thereby introducing variable action durations into the planning model.
	This contrasts with the classical MAPF formulation, which assumes unit edge costs and actions that complete in a single timestep. While the classical MAPF formulation simplifies computation, it often fails to capture the variability in edge lengths found in real-world environments. In MAPF$_Z$, although time remains discretized, agents may spend multiple timesteps on a single edge, introducing more complex temporal interactions.
	
	Let $A$ denote the set of agents, where $|A| < |V|$, ensuring that each agent is assigned a unique start and goal location. The initial and goal positions are specified by two injective mappings $\mathcal{S}, \mathcal{T}: A \to V$, where $\mathcal{S}_i = \mathcal{S}(a_i)$ and $\mathcal{T}_i = \mathcal{T}(a_i)$ denote the start and goal vertices of agent $a_i$, respectively. The injectivity of these mappings ensures that no two agents share the same start or goal location.
	A plan for agent $i$ is represented as a sequence $\pi_i = \{\langle v_i^1,t_i^1 \rangle,\ldots,\langle v_i^{T_i},t_i^{T_i}\rangle\}$, where $v_i^{\tau} \in V$ indicates the vertex reached at time $t_i^{\tau} \geq 0$. The same vertex may appear in consecutive elements of the sequence, i.e., $v_i^{\tau}=v_i^{\tau-1}$, which represents a wait action at that vertex. The plan $\pi_i$ is considered feasible if it satisfies the following conditions: (1) the vertex sequence $[v_i^1, \ldots, v_i^{T_i}]$ forms a valid path from the start vertex $\mathcal{S}_i$ to the goal vertex $\mathcal{T}_i$ in graph $\mathcal{G}$; (2) for all $j \in \{2, \ldots, T_i\}$, the timing constraint $t_i^j \geq t_i^{j-1} + W(v_i^{j-1}, v_i^j)$ holds, meaning the agent cannot arrive at a new vertex before completing the traversal of the previous edge.
	The ultimate objective is to compute a conflict-free solution that minimizes the makespan, defined as $\max_{i=1}^{k} c(\pi_i)$, where $c(\pi_i)$ denotes the cost of the path $\pi_i$ for agent $i$ and $k$ is the number of agents.

	\subsection{Review of CBS and heuristic techniques}
	\subsubsection{Conflict-based search for MAPF}
	CBS is a two-level optimal algorithm for solving the MAPF problem. At the low level, CBS performs a best-first search for each agent to find a shortest path that satisfies the current set of constraints. At the high level, it conducts a best-first search over a binary Constraint Tree (CT), where each CT node contains a set of constraints used to avoid conflicts, along with a corresponding plan consisting of individual paths for all agents that respect these constraints. The cost of a CT node is defined as the sum of the costs of all paths in its plan. CBS operates by checking for conflicts in the current CT node’s plan; if the plan is conflict-free, the algorithm terminates with an optimal solution. Otherwise, CBS selects one of the detected conflicts and resolves it by splitting the current CT node into two child nodes. 
	In each child node, a negative constraint of the form $\overline{\langle i, v, t \rangle}$ is added to prevent one of the conflicting agents from occupying the contested vertex $v$ at time $t$. Consequently, this agent’s path must be replanned using the low-level search, while the paths of other agents remain unchanged. By generating two child CT nodes per conflict and exploring both resolutions, CBS guarantees completeness and optimality.
	
	\subsubsection{Prioritizing conflicts}
	PC heuristic \cite{c24} guides the selection of which conflict to resolve when expanding a CT node. Since different selection strategies can result in CT of varying sizes, the choice of heuristic significantly affects both the runtime and success rate of the algorithm.
	Conflicts in the PC heuristic are classified into three types: cardinal, semi-cardinal, and non-cardinal.
	A conflict is cardinal if splitting a CT node over it increases the costs of both child nodes, semi-cardinal if only one child’s cost increases, and non-cardinal if neither child’s cost increases. The PC strategy prioritizes resolving cardinal conflicts, followed by semi-cardinal and non-cardinal conflicts, significantly reducing the number of expanded CT nodes and enhancing algorithm efficiency.
	
	\subsubsection{Disjoint splitting}
	The standard conflict splitting strategy is non-disjoint, i.e., when it splits a problem into two subproblems, some solutions are shared by both subproblems, which can create duplication of search effort.
	To address this issue, DS strategy \cite{c25} introduces positive and negative constraints to handle vertex conflicts. A negative constraint $\overline{\langle i, v, t \rangle}$ means agent \(i\) must not be at \(v\) at time \(t\), while a positive constraint $\langle i, v, t \rangle$ means agent \(i\) must be at \(v\) at time \(t\).
	When splitting a CT node over a conflict, it selects one of the conflicting agents and generates two child nodes: one with a negative constraint and the other with a positive constraint.
	This DS strategy preserves the theoretical properties of CBS while reducing redundant search efforts, and it can also improve the success rate in solving instances.
	
	\section{Method}\label{sec4}
	This section presents the implementation of CBS-NIC, which incorporates time-interval-based conflict detection and constraint definition tailored to MAPF$_Z$ at the high level, and integrates an enhanced SIPP algorithm at the low level. In addition, we introduce a novel framework BOGD, and provide a theoretical analysis of its performance, with a particular focus on the asymptotic regret bound.

	\subsection{CBS-NIC solver}
	
	\subsubsection{Time-interval-based conflict detection}
	CBS-NIC is designed for MAPF$_Z$, where each action has an integer duration, i.e., the execution time of an action is an integer multiple of a unit time step. 
	Under this setting, the execution of an action naturally induces a time interval during which an agent occupies a vertex or an edge.
	
	Let \(e_i^{t_i} = (v_i^{t_i}, v_i^{t_i^e}) \in E\) denote the edge traversed by agent \(i\), where \(t_i\) and \(t_i^e\) are the start and end time steps of the traversal, respectively, with duration \(t_i^e - t_i \in \mathbb{N}\).
	The traversal induces a spatio-temporal occupancy over the half-open time interval \([t_i, t_i^e)\).
	Specifically, agent \(i\) occupies the start vertex \(v_i^{t_i}\) at time step \(t_i\), and for any time step \(t \in (t_i, t_i^e)\), it is considered to be traversing and thus occupying edge \(e_i = (v_i^{t_i}, v_i^{t_i^e})\).
	At time step \(t_i^e\), the agent reaches the target vertex \(v_i^{t_i^e}\), which is treated as the start of its subsequent occupancy.
	
	The formal definitions of conflicts are given as follows:
	\begin{definition}[Vertex Conflict]\label{def1}
		A \textit{vertex conflict} occurs when two agents \(i\) and \(j\) occupy the same vertex \(v \in V\) at the same time step \(t \in \mathbb{N}\). Formally, a vertex conflict is defined as:
		\[
		\exists v \in V, \, t \in \mathbb{N}, \, \text{s.t.} \, v_i^t = v_j^t = v,
		\]
		where \(v_i^t\) and \(v_j^t\) represent the positions of agents \(i\) and \(j\) at time step \(t\). Such a conflict is denoted by the tuple \(\langle i, j, v, t \rangle\).
	\end{definition}
	
	\begin{definition}[Edge Conflict]\label{def2}
		An \textit{edge conflict} occurs when two agents \(i\) and \(j\) traverse the same edge in opposite directions during overlapping time intervals. Formally, an edge conflict is defined as:
		\[
		\exists e_i^{t_i}, e_j^{t_j} \in E, \, \text{s.t.} \, v_i^{t_i^e} = v_j^{t_j} \, \text{and} \, v_j^{t_j^e} = v_i^{t_i},
		\]
		where \(e_i^{t_i} = (v_i^{t_i}, v_i^{t_i^e})\) and \(e_j^{t_j} = (v_j^{t_j}, v_j^{t_j^e})\) represent the edges traversed by agents \(i\) and \(j\), starting at time steps \(t_i\) and \(t_j\) and ending at time steps \(t_i^e\) and \(t_j^e\), respectively. The time intervals are considered overlapping if:
		\[
		t_i < t_j^e \leq t_i^e \quad \text{or} \quad t_i \leq t_j < t_i^e.
		\]
		Such an edge conflict is denoted by the tuple \(\langle i, j, e_i^{t_i}, e_j^{t_j} \rangle\).
	\end{definition}
	
	Fig. \ref{fig2}(a) illustrates a 4-agents MAPF$_Z$ instance. The dashed borders indicate each agent's start location, and the same-color borders mark their destinations. For agents \(i\) and \(j\), an edge conflict occurs on edge L-M. 
	Fig. \ref{fig2}(b) shows the unsafe interval of edge conflict between agents \(i\) and \(j\). The time interval for agent \(i\) to cross edge \(e_i^{t_i} = (v_i^{t_i}, v_i^{t_i^e})\) is \((t_i, t_i^e)\), and for agent \(j\) to cross edge \(e_j^{t_j} = (v_j^{t_j}, v_j^{t_j^e})\) is \((t_j, t_j^e)\). The unsafe interval is when the agents simultaneously traverse the edge. This occurs in two cases:
	1) Agent \(j\) starts before agent \(i\) finishes, i.e., \(t_i < t_j^e \leq t_i^e\).
	2) Agent \(j\) starts after agent \(i\) has already started, i.e., \(t_i \leq t_j < t_i^e\).
	If \(t_j^e = t_i^e\) or \(t_i = t_j\), it indicates that both agents traverse the edge simultaneously. If \(t_i = t_j^e\) or \(t_j = t_i^e\), it implies a vertex conflict at either L or M.

	\begin{figure}
		\centering
		\includegraphics[width=3.5in]{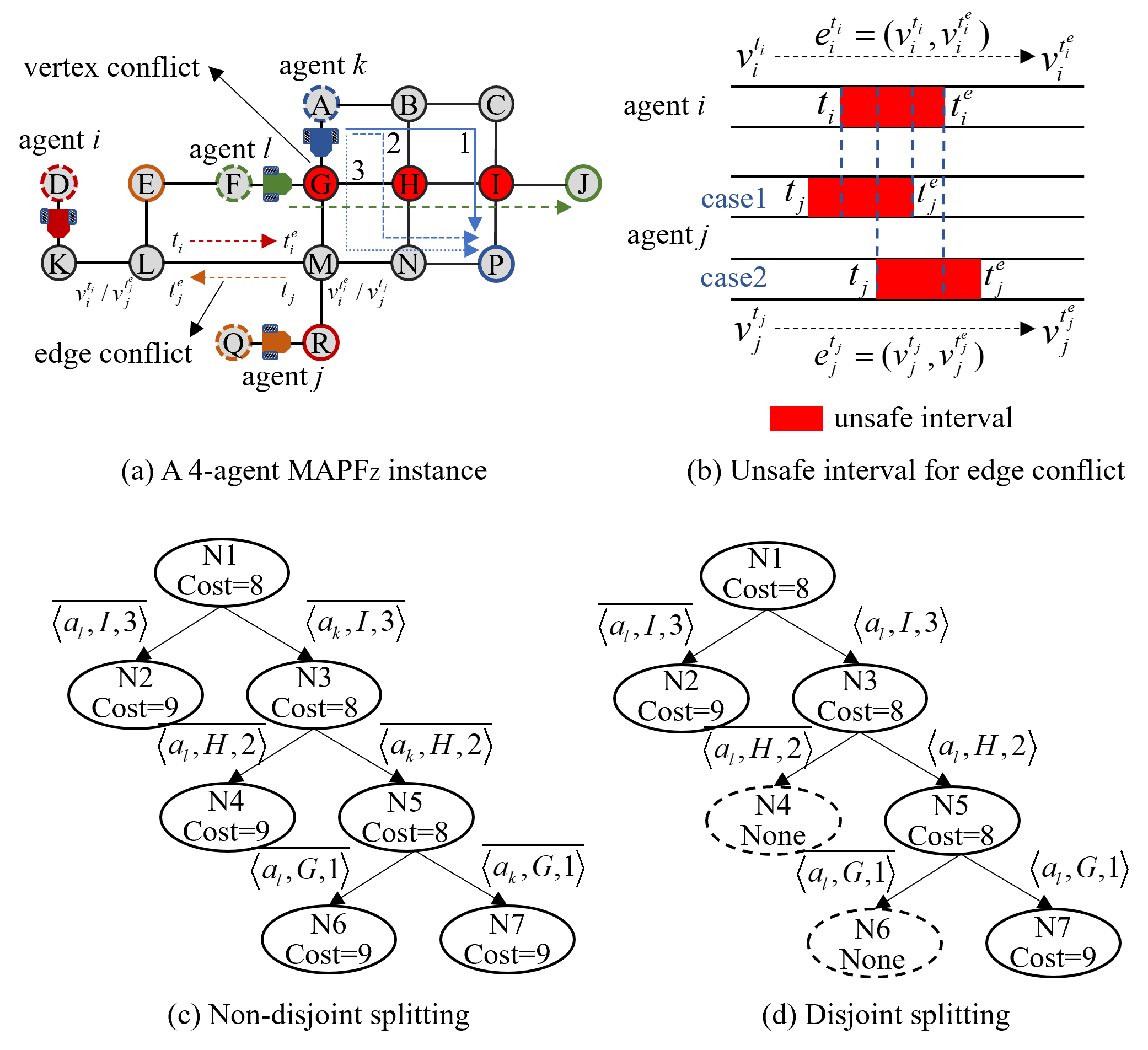}
		\caption{Conflict detection and constraint addition of CBS-NIC framework.}
		\label{fig2}
	\end{figure}

	\subsubsection{Time-interval-based constraints definition}
	In CBS-NIC, the high-level search selects the minimum cost leaf node \( N \) and performs conflict detection on its solution \( N.\Pi \). If no conflicts are found, \( N \) is the valid solution, and the search halts. Otherwise, the high-level search expands \( N \) to resolve the conflict by generating two new CT nodes, \( N_i \) and \( N_j \). To compute the constraints to add to \( N_i \) and \( N_j \), CBS-NIC determines the unsafe intervals for each vertex or edge. The specific constraints are defined as follows: 
	For a vertex conflict \(\langle i, j, v, t \rangle\), the added constraints are designed to prevent agent \(i\) from being at vertex \(v\) at time step \(t\), or to prevent agent \(j\) from being at vertex \(v\) at time \(t\). As a result, the unsafe time for vertex \(v\) is time step \(t\). This can be mathematically expressed as \(\overline{\langle i, v, t \rangle}\) and \(\overline{\langle j, v, t \rangle}\).
	For an edge conflict \(\langle i, j, e_i^{t_i}, e_j^{t_j} \rangle\), the added constraints are to prevent agent \(i\) from occupying edge \(e_i^{t_i}\) during the time interval \((t_j, t_j^e)\), or to prevent agent \(j\) from occupying edge \(e_j^{t_j}\) during the time interval \((t_i, t_i^e)\). This can be mathematically expressed as: 
	\(\overline{\langle i, e_i^{t_i}, (t_j, t_j^e) \rangle}\) and \(\overline{\langle j, e_j^{t_j}, (t_i, t_i^e) \rangle}\).
	
	\subsubsection{Disjoint splitting for CBS-NIC}
	To reduce the expansion of CT nodes, we introduce the DS technique, which incorporates both positive and negative constraints. For a conflict \(\langle i, j, v, t \rangle\), CBS-NIC-DS generates two child CT nodes and adds a negative constraint $\overline{\langle i, v, t \rangle}$ to one and a positive constraint \(\langle i, v, t \rangle\) to the other. 
	Negative constraints $\overline{\langle i, v, t \rangle}$ indicate that agent \(i\) cannot occupy vertex \(v\) at time step \(t\) (all previously introduced constraints are negative constraints). On the other hand, positive constraints \(\langle i, v, t \rangle\) specify that agent \(i\) must occupy vertex \(v\) at time step \(t\). Positive constraints are composed of multiple negative constraints: for agent \(i\), it cannot reach any vertices other than \(v\) at time \(t\), and for other agents (\(\forall j \in A, j \neq i\)), they cannot reach vertex \(v\) at time step \(t\).
	The definition of a positive constraint is as follows:
	\begin{definition}[Positive Constraint]\label{def3}
		For the positive constraint \(\langle i, v, t \rangle\), it is equivalent to the union of two sets of negative constraints. 
		The first set of negative constraints \(C_i\) ensures that agent \(i\) cannot occupy any vertex other than \(v\) at time \(t\), which is mathematically represented as:
		\[
		C_i = \{\overline{\langle i, v', t \rangle} \mid v' \neq v, v' \in V \}.
		\]
		The second set of negative constraints \(C_{\neg i}\) ensures that no other agent can occupy vertex \(v\) at time \(t\), which is represented as:
		\[
		C_{\neg i} = \{\overline{\langle j, v, t \rangle} \mid j \neq i, j \in A \}.
		\]
		The positive constraint for agent \(i\) to occupy vertex \(v\) at time \(t\) is the union of these two sets:
		\[
		C_p = C_i \cup C_{\neg i}.
		\]
	\end{definition}
	
	In Fig. \ref{fig2}(a), for agents \(k\) and \(l\), a vertex conflict arises at one of the vertices G, H, or I.
	In Fig. \ref{fig2}(c), CBS-NIC with non-disjoint splitting resolves conflicts by generating multiple CT nodes for each conflict, resulting in redundant nodes and higher computational costs. In contrast, Fig. \ref{fig2}(d) demonstrates CBS-NIC-DS with disjoint splitting, which prunes infeasible nodes and reduces the number of CT nodes, achieving the same solution with fewer leaf nodes and improved efficiency.
	
	\subsubsection{SIPP for the low-level of CBS-NIC}
	The low-level solver of CBS-NIC is built on constrained Safe Interval Path Planning (SIPP) \cite{c26}. SIPP computes a set of safe intervals for each location $v \in V$, where a safe interval is defined as the maximal contiguous time interval during which an agent can occupy or arrive at $v$ without colliding with moving obstacles. Unlike the original SIPP, which minimizes the total duration of action sequences in a continuous-time domain, the enhanced SIPP operates in a discrete-time domain and directly incorporates integer edge weights. This adaptation ensures compatibility with MAPF$_Z$, where edge traversal costs are explicitly defined as integers.
	
	To adapt SIPP as the low-level solver for CBS-NIC, it is modified to prohibit actions that violate the given constraints, ensuring compliance with the conflict resolution process.
	For a vertex conflict constraint \(\overline{\langle i, v, t \rangle}\), the agent is prohibited from occupying vertex \(v\) at time step \(t\). This is achieved by splitting the safe intervals of \(v\) accordingly. For instance, if \(v\) initially has a single safe interval \([0, \infty)\), it is split into two intervals: \([0, t)\) and \((t, \infty)\), ensuring the agent cannot occupy \(v\) during the restricted time step.
	For an edge conflict constraint \(\langle i, e_i^{t_i}, (t_j, t_j^e) \rangle\), agent \(i\) is forbidden from traversing edge \(e_i^{t_i} = (v_i^{t_i}, v_i^{t_i^e})\) during the time interval \((t_j, t_j^e)\). If the planned traversal by agent \(i\) overlaps with this interval, i.e., the agent attempts to move from \(v_i^{t_i}\) to \(v_i^{t_i^e}\) starting at any time \(t \in (t_j, t_j^e)\), the action is considered invalid. 
	To enforce this, the safe intervals of edge \(e_i^{t_i}\) are adjusted by removing the interval \((t_j, t_j^e)\), which represents the conflict period. 
	To satisfy the constraint, the agent must wait at \(v_i^{t_i}\) until the start of the next safe interval, i.e., until \(t \geq t_j^e\), before performing the traversal.

	\begin{algorithm}
		\BlankLine
		\KwIn{A tuple $\langle \mathcal{G},S,G \rangle$}    
		\KwOut{The conflict-free solution.}
		Initialize path set \(\Pi \gets \emptyset\)\;
		\ForEach{agent \(i\)}{
			\(\pi_i \gets SIPP(\mathcal{G}, S_i, G_i, \emptyset)\)\;
			Append \(\pi_i\) to \(\Pi\) \;
		}
		conflictNodes $\gets getPCnodes(\Pi,\emptyset)$\;
		$start \gets (\emptyset$, conflictNodes, $\Pi$, cost($\Pi$))\;
		Create OPEN and add $start$ to OPEN\;
		\While{OPEN is not empty}{\label{line7}
			$N \gets$ OPEN.$pop()$\;
			\If{$N.conflictNodes$ is empty}{\label{line9}
				\textbf{return} $N.\Pi$\label{line10}
			}
			$c \gets getHighestPCnode$(N.conflictNodes)\;
			\For{n \textbf{in} $c.nodes$}{ 
				Create leaf node newNode of root $N$\;
				newNode.$c_p$ $\gets$ n.constraints\; 
				newNode.$\Pi \gets$ n.$\Pi$\; 
				newNode.cost $\gets$ cost(n.$\Pi$)\;
				newNode.conflictNodes $\gets getPCnodes$(newNode.$\Pi$, newNode.$c_p$)\;  
				OPEN.push(newNode)\;  
				}
		}	
		\textbf{return} False\label{line20}
		\caption{CBS-NIC pseudo code.}
		\label{alg1}
	\end{algorithm}
	
	\begin{algorithm}
		\BlankLine
		\KwIn{A tuple $\langle \mathcal{G},S,G \rangle$, constraints $C_{all}$, paths $\Pi$.}    
		\KwOut{The set of conflict nodes conflictNodes.}
		conflictNodes = \{\}\;
		\For{$i \gets 1, j \gets i + 1$ \textbf{to} $k$}{
			\For{$t \gets 1$ \textbf{to} $T_{max}$}{
				Get $t_i,t_i^e,t_j,t_j^e$ according to $\pi_i, \pi_j$\;
				// Vertex conflict\;
				\If{$v_i^{t_i} = v_j^{t_j}$}{
					conflict = $\left \langle i,j,v_i^{t_i},t\right \rangle$\;
					$C_p = \{\left \langle i,v_i^{t_i},t\right \rangle,\overline{\left \langle i,v_i^{t_i},t\right \rangle}\}$\;
					\textbf{break}\;
				}
				// Edge conflict\;
				\ElseIf{($v_i^{t_i} = v_j^{t_j^e}$ and $v_i^{t_i^e} = v_j^{t_j}$) and ($t_i < t_j^e \leq t_i^e$ or $t_i \leq t_j < t_i^e$)}{
					conflict = $\langle i, j, e_i^{t_i}, e_j^{t_j} \rangle$\;
					$C_p = \{\langle i, e_i^{t_i}, (t_j, t_j^e) \rangle, \langle j, e_j^{t_j}, (t_i, t_i^e) \rangle\}$\;
					\textbf{break}\;
				}
			}
			\If{conflict is empty}{\textbf{continue}\;}
			conflictNode.conflict = conflict\;
			\For{$c_p$ \textbf{in} $C_p$}{
				n.constraints = $c_p \cup C_{all}$\;
				n.$\Pi$ = SIPP($\mathcal{G},S,G,n.constraints$)\;
				conflictNode.nodes.$push$(n)\;
			} 
			conflictNodes.$push$(conflictNode)\;
		}
		Use PC strategy to prioritize nodes in conflictNodes\;
		\textbf{return} conflictNodes
		\caption{Get priority conflict node pseudo code.}
		\label{alg2}
	\end{algorithm}

	\subsubsection{CBS-NIC pseudo-code}
	The pseudo code for the CBS-NIC algorithm is shown in Algorithm \ref{alg1}. 
	Initially, the algorithm starts by constructing an initial path set \(\Pi\). For each agent \(i\), it computes the initial path \(\pi_i\) using the SIPP algorithm and adds it to the path set \(\Pi\). The input of SIPP includes the graph \(\mathcal{G}\), start location \(S_i\), goal location \(G_i\), and the constraint set $\emptyset$ (lines 1-4).
	Next, conflict nodes, which contain all conflicts, their corresponding constraints, and the paths under those constraints, are generated using \(getPCnodes \) (Algorithm \ref{alg2}). The initial conflict tree (CT) node is created, including the constraint set $\emptyset$, conflict nodes, the path set \(\Pi\), and the cost of path. This initial CT node is then added to the OPEN list for further processing (lines 5-7). 
	The main loop continues until the OPEN list is empty. At each iteration, the node \(N\) with the minimal cost is extracted from the OPEN list. If no conflict nodes are found in node \(N\), the algorithm returns the current set of paths \(N.\Pi\) (lines 8-11). Otherwise, the algorithm selects a conflict node 
	\(c\) using \(getHighestPCnode\), which prioritizes cardinal conflicts over semi-cardinal conflicts, and semi-cardinal conflicts over non-cardinal conflicts. It then iterates through the nodes in \(c.\text{nodes}\) involved in the conflict, creating a leaf node \(newNode\) under the root CT node \(N\). The constraint \(c_p\) of \(newNode\) is set to \(n.\text{constraint}\), the path \(\Pi\) of \(newNode\) is \(n.\Pi\), and the cost of \(newNode\) is $cost(n.\Pi)$. 
	The conflict nodes of \(newNode\) are generated using Algorithm \ref{alg2}, with the input being the path and constraint of \(newNode\). The \(newNode\) is then added to the OPEN list. If no solution is found, the algorithm returns False (lines 12-19).
	
	The pseudo-code for obtaining the priority conflict node is outlined in Algorithm \ref{alg2}. The algorithm takes as input a tuple \(\langle \mathcal{G}, S, G \rangle\), representing the graph, start, and goal locations for the agents, along with the global constraint set \(C_{all}\) and the current paths \(\Pi\). An empty list is initialized to store detected conflicts and their associated nodes. 
	The algorithm iterates over all agent pairs \(i, j\) and examines their paths at each timestep up to \(T_{max}\). 
	For each pair of agents, the algorithm first computes the start timestep $t_i,t_j$ and end timestep $t_i^e,t_j^e$ of the agent $i$ and $j$ crossing the edge $e$ at time step $t$ according to their paths $\pi_i,\pi_j$ (lines 1-4).
	It then first checks for a vertex conflict as defined in Definition \ref{def1}. If a vertex conflict exists, the corresponding constraint set 
	$C_p = \{\left \langle i,v_i^{t_i},t\right \rangle,\overline{\left \langle i,v_i^{t_i},t\right \rangle}\}$
	is added. Otherwise, it checks for an edge conflict as defined in Definition \ref{def2}. If an edge conflict is found, the corresponding constraint set $C_p = \{\langle i, e_i^{t_i}, (t_j, t_j^e) \rangle, \langle j, e_j^{t_j}, (t_i, t_i^e) \rangle\}$ is added (lines 6-14).
	After detecting a conflict, the algorithm creates a conflict node, associates it with the constraints \(C_p\) and \(C_{all}\), computes the updated path \(\Pi\) for the node using SIPP, and adds the node to the list of conflict nodes. The conflicts are categorized as cardinal, semi-cardinal, and non-cardinal to prioritize resolution. Finally, the algorithm returns conflictNodes(lines 15-24)
	
	\subsection{Bayesian Optimization for Graph Design}
	
	\subsubsection{The framework of bayesian optimization}
	
	\begin{figure}
		\centering
		\includegraphics[width=3.5in]{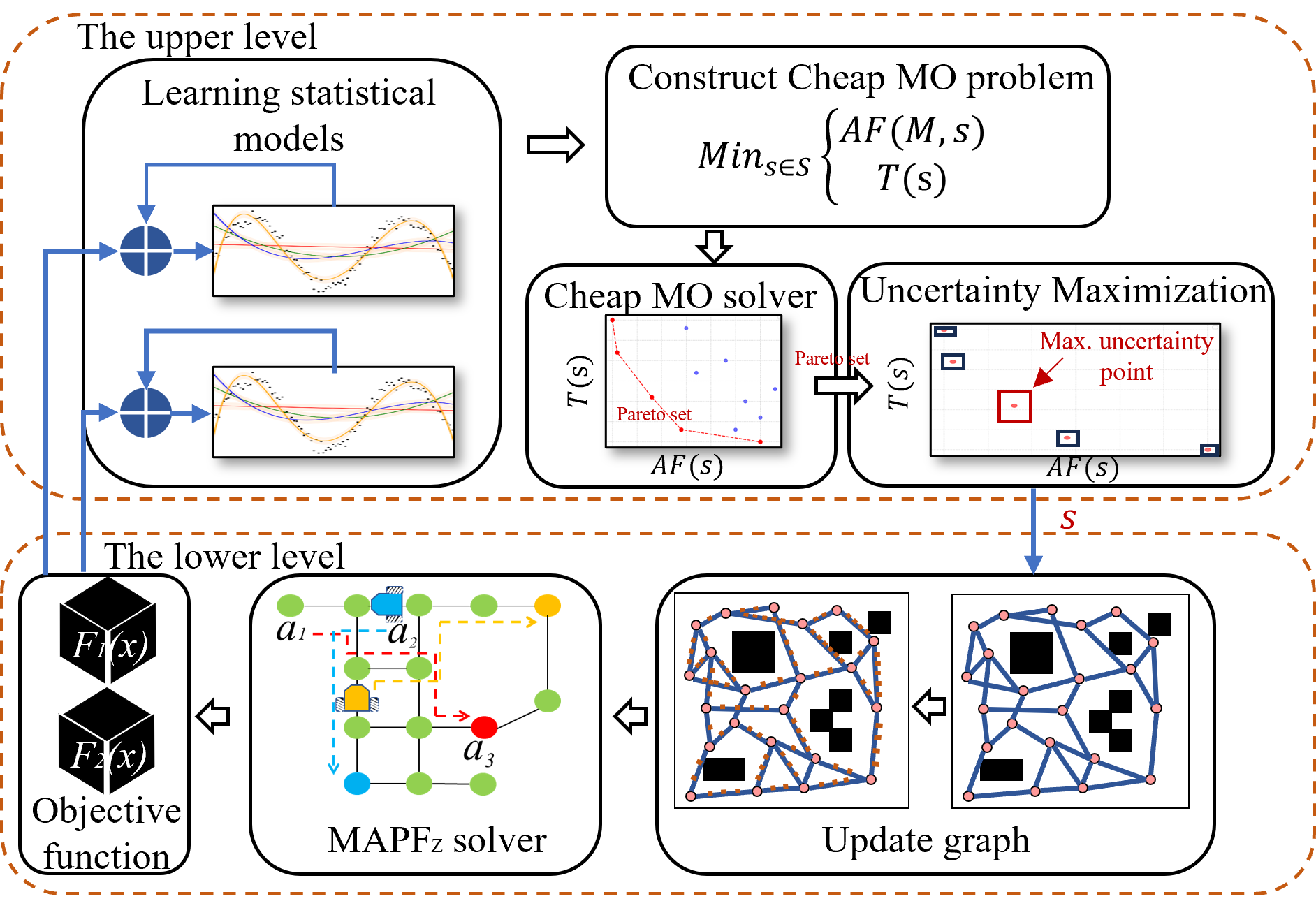}
		\caption{The framework of BOGD.}
		\label{fig3}
	\end{figure}
	
	The graph discretization problem is formulated as a multi-objective optimization task that seeks to balance computational efficiency with discretization accuracy. Let \( s \) denote the discretization parameter. The first objective is the algorithm runtime \( T(s) \), modeled as an expensive black-box function. The second objective is the discretization error, defined as $C(s) = \sum_{i=1}^{k} \sum_{e \in \pi_i} \left| w_e - \operatorname{round}\left( \frac{w_e}{s} \right) \cdot s \right|$,
	where \( w_e \) is the original edge weight, \( \pi_i \) is the path of agent \( i \), and \( k \) is the total number of agents. The error measures the cumulative deviation between original and discretized weights along the paths of all agents.
	
	To tackle this problem, the BOGD framework, illustrated in Fig.~\ref{fig3}, adopts a two-level hierarchical structure. 
	At the lower level, for any candidate discretization parameter \( s \), the original real-valued edge weights \( w_i \in \mathbb{R}_{>0} \) are transformed into integer weights by computing \( w_i' = \operatorname{round}(w_i / s) \). This transformation enables the application of CBS-NIC, which operates on integer-weighted graphs, while implicitly maintaining an approximation to the original costs with a bounded discretization error.
	At the upper level, BOGD searches for the optimal discretization parameter \( s \). A surrogate model \( M \) is built to approximate the unknown runtime function \( T(s) \) based on historical observations. The error function \( C(s) \), being deterministic and known, is incorporated directly. A multi-objective optimization (MO) problem is then formulated using the acquisition function (AF) derived from the surrogate model and the known error function.
	
	To initialize the optimization process, a biased initial population is generated for a genetic algorithm. This population is evaluated using the score function \( \widetilde{F}(s) = AF(M, s) + C(s) \), which favors candidates with both high potential and high uncertainty, the latter being captured by the AF. NSGA-II \cite{c26}, a widely adopted multi-objective genetic algorithm, is employed to evolve the population. The algorithm ranks solutions via fast non-dominated sorting and applies crossover and mutation operations to efficiently explore the search space.
	After obtaining a Pareto front from the surrogate multi-objective problem, BOGD selects a candidate \( s \) that minimizes \( \widetilde{F}(s) \) from the front to be evaluated by the true black-box function. This strategy promotes exploration by prioritizing solutions with high uncertainty, thereby improving model accuracy and facilitating convergence toward the true Pareto front.
	The acquisition function used in this framework follows the Lower Confidence Bound (LCB) principle:
	\[
	LCB(s) = \mu^*(s) - \sqrt{2 \log\left( t^{d/2+2} \pi^{2}/3\delta \right)} \cdot \sigma^*(s),
	\]
	where \( \mu^*(s) \) and \( \sigma^*(s) \) denote the posterior mean and standard deviation from the surrogate model, \( t \) is the current iteration, \( d \) is the dimensionality of the design space (here, \( d = 1 \)), and \( \delta \in (0,1) \) determines the confidence level\cite{c27}. This acquisition function facilitates the trade-off between exploitation and exploration.

	\subsubsection{Regret bound}
	
	In the MAPF problem, the objective is to optimize the discretization parameter \( s \in [s_{\min}, s_{\max}] \) to balance the algorithm runtime \( T(s) \), and the discretization error \( C(s) \). The BOGD framework addresses this multi-objective optimization problem by iteratively selecting \( s_t \) at each iteration \( t \). This subsection derives a sublinear regret bound for the cumulative regret \( R(s^*) \), ensuring the theoretical guarantees of the BOGD approach.
	
	\begin{theorem}
		\label{thm:regret_bound}
		Let \( s^* \in S^* \) be a solution in the Pareto optimal set, and \( s_t \in S_t \) be the solution obtained by BOGD at iteration \( t \). Define the total regret as \( R(s^*) = \|R_T(s^*), R_C(s^*)\| \) using the maximum norm, where \( R_T(s^*) = \sum_{t=1}^{t_m} (T(s_t) - T(s^*)) \) and \( R_C(s^*) = \sum_{t=1}^{t_m} (C(s_t) - C(s^*)) \). Then, with probability at least \( 1 - \delta \), the regret is bounded as:
		\[
		R(s^*) \leq \sqrt{d t_m \beta_{t_m} \gamma_{t_m}} + W D \sqrt{t_m},
		\]
		where \( d = 1 \) since \( s \) is one-dimensional, \( \beta_{t_m} = 2 \log\left(\frac{\pi^2 t_m^2}{6 \delta}\right) \) with \( \delta \in (0, 1] \) as the confidence parameter, \( \gamma_{t_m} \) is the maximum information gain of \( T(s) \) after \( t_m \) iterations, \( D = s_{\max} - s_{\min} \) is the diameter of the decision space, \( t_m \) is the maximum number of iterations, and \( W = \sum_{i=1}^k \sum_{e \in \pi_i} \frac{w_e}{s_{\min}} \).
	\end{theorem}
	
	\begin{proof}
		The regret bound is derived by separately bounding \( R_T(s^*) \) and \( R_C(s^*) \), then combining them to obtain the total regret under the maximum norm.
		
		Step 1: Bounding \( R_T(s^*) \). The runtime objective \( T(s) \) is modeled as a Gaussian process (GP). We use the lower confidence bound (LCB) acquisition function:
		\[
		LCB_t(s) = \mu_{t-1}(s) - \sqrt{2 \log K_t} \cdot \sigma_{t-1}(s),
		\]
		where \( \mu_{t-1}(s) \) and \( \sigma_{t-1}(s) \) denote the posterior mean and standard deviation of the GP at iteration \( t-1 \), and \( K_t = \frac{\pi^2 t^2}{6 \delta} \). According to \cite{c27}, with probability at least \( 1 - \delta \):
		\[
		|T(s) - \mu_{t-1}(s)| \leq \sqrt{2 \log K_t} \cdot \sigma_{t-1}(s).
		\]
		Assuming \( s_t \) is selected to minimize \( LCB_t(s) \) within the Pareto set \( S_t \), either \( LCB_t(s_t) \leq LCB_t(s^*) \) or \( s_t = s^* \). By \cite{c28}, this implies:$LCB_t(s_t) \leq LCB_t(s^*) \leq T(s^*).$
		Thus:
		\begin{align*}
			T(s_t) - T(s^*) &\leq T(s_t) - LCB_t(s_t) \\
			&\leq T(s_t) - \mu_{t-1}(s_t) + \sqrt{2 \log K_t} \sigma_{t-1}(s_t), \\
			&\leq 2 \sqrt{2 \log K_t} \sigma_{t-1}(s_t),
		\end{align*}
		Summing over \( t_m \) iterations and applying results from \cite{c27}, the cumulative regret for \( T(s) \) is:$
		R_T(s^*) \leq \sqrt{C_1 d t_m \beta_{t_m} \gamma_{t_m}},$
		where \( d = 1 \), \(C_1 = 8/log(1 + \sigma^{-2})\), \( \beta_{t_m} = 2 \log\left(\frac{\pi^2 t_m^2}{6 \delta}\right) \), and \( \gamma_{t_m} \) is the maximum information gain.
		
		Step 2: Bounding \( R_C(s^*) \). The error objective \( C(s) = \sum_{i=1}^k \sum_{e \in \pi_i} \left| w_e - \operatorname{round}\left( \frac{w_e}{s} \right) \cdot s \right| \) is deterministic. We establish its Lipschitz continuity. For a single term \( c_e(s) = \left| w_e - \operatorname{round}\left( \frac{w_e}{s} \right) \cdot s \right| \), within an interval where \( \operatorname{round}\left( \frac{w_e}{s} \right) = m \), we have:$
		c_e(s) = \left| w_e - m \cdot s \right|.$
		The derivative is \( \pm m \), and since \( m \leq \frac{w_e}{s} \), for \( s \geq s_{\min} \), the slope is bounded by \( \frac{w_e}{s_{\min}} \). The Lipschitz constant for \( C(s) \) over all edges is:$
		L \leq \sum_{i=1}^k \sum_{e \in \pi_i} \frac{w_e}{s_{\min}}.$
		The cumulative regret for \( C(s) \) is:
		\begin{align*}
			R_C(s^*) &= \sum_{t=1}^{t_m} (C(s_t) - C(s^*)) \\
			&\leq \sum_{t=1}^{t_m} |C(s_t) - C(s^*)| \\
			&\leq L \cdot \sum_{t=1}^{t_m} |s_t - s^*|.
		\end{align*}
		Since \( s_t, s^* \in [s_{\min}, s_{\max}] \), \( |s_t - s^*| \leq D \), where \( D = s_{\max} - s_{\min} \). In Bayesian optimization, the sum of distances grows sublinearly:$
		\sum_{t=1}^{t_m} |s_t - s^*| \leq D \sqrt{t_m}$ \cite{bubeck2015convex}.
		Thus:
		\[
		R_C(s^*) \leq \left( \sum_{i=1}^k \sum_{e \in \pi_i} \frac{w_e}{s_{\min}} \right) D \sqrt{t_m}.
		\]
		
		Step 3: Total Regret Bound. The total regret \( R(s^*) = \| [R_T(s^*), R_C(s^*)] \|_\infty \) is bounded by:
		\[
		R(s^*) \leq \max\left( \sqrt{d t_m \beta_{t_m} \gamma_{t_m}}, W D \sqrt{t_m} \right).
		\]
		A conservative bound is:
		\[
		R(s^*) \leq \sqrt{d t_m \beta_{t_m} \gamma_{t_m}} + W D \sqrt{t_m}.
		\]
		Both terms scale as \( O(\sqrt{t_m \log t_m}) \), ensuring sublinear regret.
		
	\end{proof}
	
	Theorem \ref{thm:regret_bound} guarantees that as the number of iterations increases, the regret converges to zero at a sublinear rate. This property is critical for optimization problems, as it indicates that BOGD becomes increasingly efficient and effective over time \cite{c29}. In graph design applications, where numerous evaluations may be required to optimize edge weights, the sublinear regret bound highlights the scalability and robustness of BOGD, making it a valuable tool for real-world graph optimization scenarios \cite{c30}.
	
	\subsection{Theoretical Analysis of CBS-NIC}
	
	We analyze the optimality and completeness of CBS-NIC for solving the MAPF$_Z$ problem on a weighted graph $G=(V,E,W)$ with positive integer edge weights. Time is discretized into integer timesteps. The objective is to minimize the makespan
	$cost(\Pi) = \max_i c(\pi_i),$
	where $c(\pi_i)$ is the arrival time of agent $i$ at its goal. The low-level planner SIPP is complete and optimal under a given set of constraints\cite{c26}.
	
	\begin{theorem}[Optimality and Completeness of CBS-NIC]
		If a conflict-free solution exists, CBS-NIC will find one and it will be a minimum-makespan solution.
	\end{theorem}
	
	\begin{proof}
		Optimality: Let \( \Pi^* \) be an optimal solution with makespan \( M^* \). Conflicts are detected and resolved by branching with interval-based constraints that forbid one conflicting action. As a result, there exists a branch of the conflict tree whose accumulated constraints are consistent with \( \Pi^* \). Since the low-level SIPP planner computes shortest paths under the imposed constraints, the corresponding CT node \( N^* \) reconstructs \( \Pi^* \) and satisfies \( cost(N^*) = M^* \).
		The high-level search expands CT nodes in non-decreasing order of \( cost(N) \) using best-first search. This implies that all nodes with a cost less than or equal to \( M^* \) will be expanded before any node with cost greater than \( M^* \). Furthermore, adding constraints cannot decrease the makespan, because the constrained agent’s feasible path set is reduced, and SIPP always returns the shortest feasible path. Therefore, \( cost(N') \ge cost(N) \) for every child node \( N' \) of \( N \).
		Now suppose the algorithm returns a conflict-free solution with cost \( M > M^* \). In this case, the node \( N^* \) with cost \( M^* \) must either still be in OPEN or have been expanded earlier. Since \( M^* < M \), best-first search would have expanded \( N^* \) before any node of cost \( M \), leading to a contradiction. Thus, CBS-NIC will first expand the optimal conflict-free node \( N^* \) with makespan \( M^* \), ensuring that it is optimal with respect to makespan.
		
		Completeness: 
		CBS-NIC is complete because it systematically explores the CT using best-first search and resolves conflicts with sound pairwise branching constraints. When two agents conflict at a vertex (or edge) during a specific time interval, the algorithm branches by adding mutually exclusive constraints: one forbidding the first agent from using the resource in that interval, and the other forbidding the second agent. This branching is sound because, in any valid conflict-free solution, at least one agent must avoid the resource during the conflicting time interval. Thus, each valid solution is preserved in at least one of the child nodes and never eliminated from both branches.
		Since the graph, the number of agents, and the discrete time horizon up to \( M^* \) are finite, the total number of possible interval-based constraints is also finite. This implies that the portion of the CT with nodes having cost \( \leq M^* \) is finite. The high-level search expands nodes in non-decreasing cost order, ensuring that every branch that could contain a valid solution is preserved due to the soundness of the branching constraints.
		If a conflict-free solution exists, a CT node that satisfies the constraints of that solution, with individual paths (computed by SIPP) realizing it without conflicts, will eventually be expanded. When this node is conflict-free, it is returned as a valid solution. Therefore, CBS-NIC is guaranteed to find a valid solution whenever one exists.
		
	\end{proof}

	\section{Experiments}\label{sec5}
	In this section, we designed evaluation experiments to assess the effectiveness of the proposed algorithm in different scenarios. The simulator was developed in C++, and the experiments were run on a laptop with Intel Core i7 2.10 GHz CPU and 16GB RAM.
	
	\subsection{Experimental setup}
	In our experiments, we utilized graphs from the recently introduced MAPF benchmark suite, available in the Moving AI repository \cite{c32}. This benchmark includes two types of graphs: $2^k$-neighborhood grids and roadmap-based graphs. These graphs cover a wide range of scenarios, including video game maps (e.g., den520d), warehouse logistics (warehouse-20-40-10-2-2), randomly generated environments (random-64-64-10), field robotics (empty-16-16), and urban maps (Berlin-1-256, Boston-0-256).
	To evaluate performance on graphs with non-unit edge costs, we applied the following procedure: For each grid map, we created a corresponding graph where each vertex represents a grid cell, and edges were added between each cell and its $2^k$ nearest neighbors, with $k$ ranging from 3 to 5. 
	For the roadmap-based environments, we employ Constrained Delaunay Triangulation (CDT) \cite{c34} to construct roadmaps in three representative map environments: Berlin-1-256, random-64-64-10, and den520d. As shown in Fig.~\ref{fig4}, each roadmap differs in terms of the number of nodes and edges, reflecting the diversity of spatial layouts and obstacle distributions across the selected maps. 
	By generating roadmaps using CDT with integrated obstacle constraints and inflation techniques, we construct consistent yet topologically diverse testing environments. This enables a comprehensive evaluation of the scalability and adaptability of our approach across varying levels of environmental complexity.
	For each graph, 25 scenario files were used, each containing randomly selected, non-overlapping start-goal pairs sampled from the graph vertices. To ensure a fair and consistent comparison, all evaluated methods are tested on the identical set of task instances for each scenario.
	\begin{figure}
		\centering
		\includegraphics[width=3.5in]{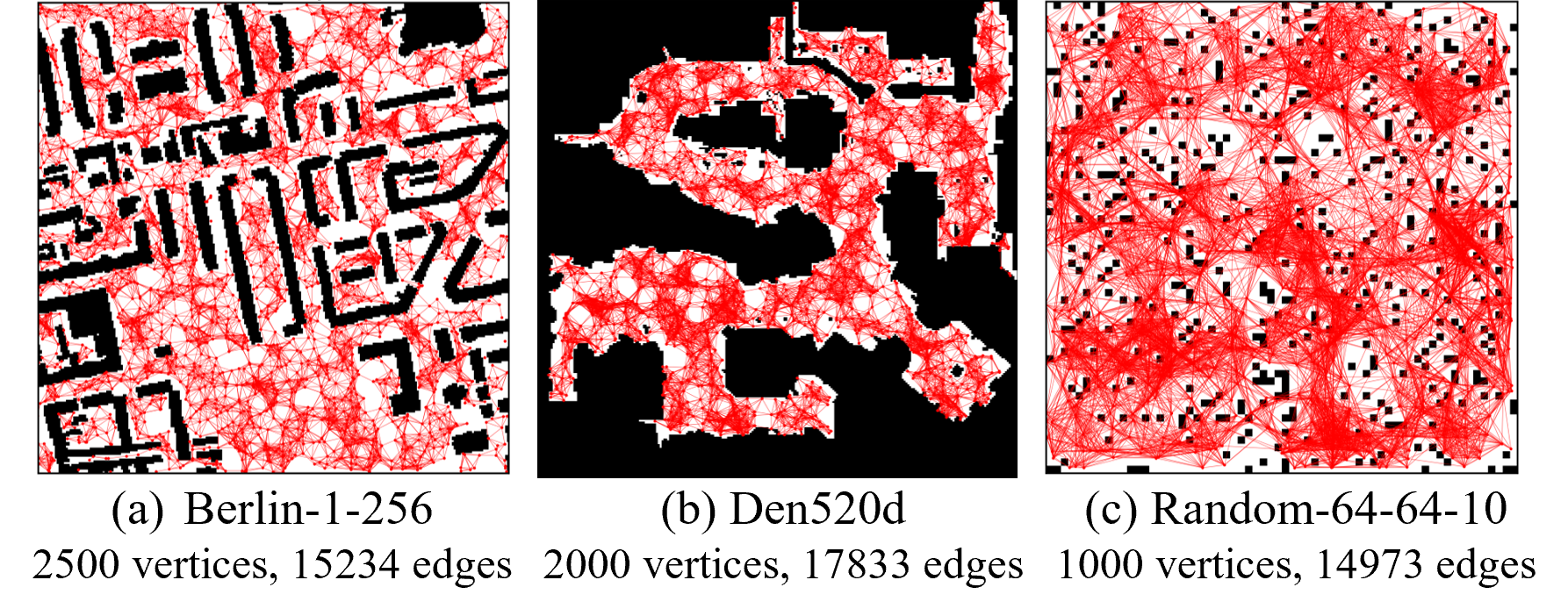}
		\caption{The roadmap visualization includes varying numbers of vertices and edges.}
		\label{fig4}
	\end{figure}
	
	We evaluate and compare the proposed CBS-NIC against several closely related algorithms, including CCBS \cite{c6}, CBS-NIC with BOGD strategy (CBS-NIC-B), Baeline, Meta-agent CBS (MA-CBS) \cite{c23} and Enhanced Partial Expansion A* (EPEA) along with its restart variant EPEA-restart \cite{c22}. The parameters $c_3$, $c_4$, and $c_5$ denote the number of neighbors used in the $2^k$-neighborhood grids, corresponding to $k = 3, 4, 5$, respectively. For the other algorithms, we report results only for $c_5$ to avoid redundancy, as it serves as a representative case that sufficiently reflects relative performance differences across methods.
	It is worth noting that MA-CBS and EPEA (including EPEA-restart) are designed specifically for grid-based MAPF and operate under different neighborhood structures.Therefore, their implementations are limited to different grid neighborhoods rather than roadmap environments.
	In CCBS and its variants, agents are represented as disks with a radius of 0.5.
	The Baseline method adopts a simplified graph construction strategy, where continuous edge costs are directly discretized via integer rounding.

	\subsection{Results}

	\begin{figure}
		\centering
		\includegraphics[width=3.5in]{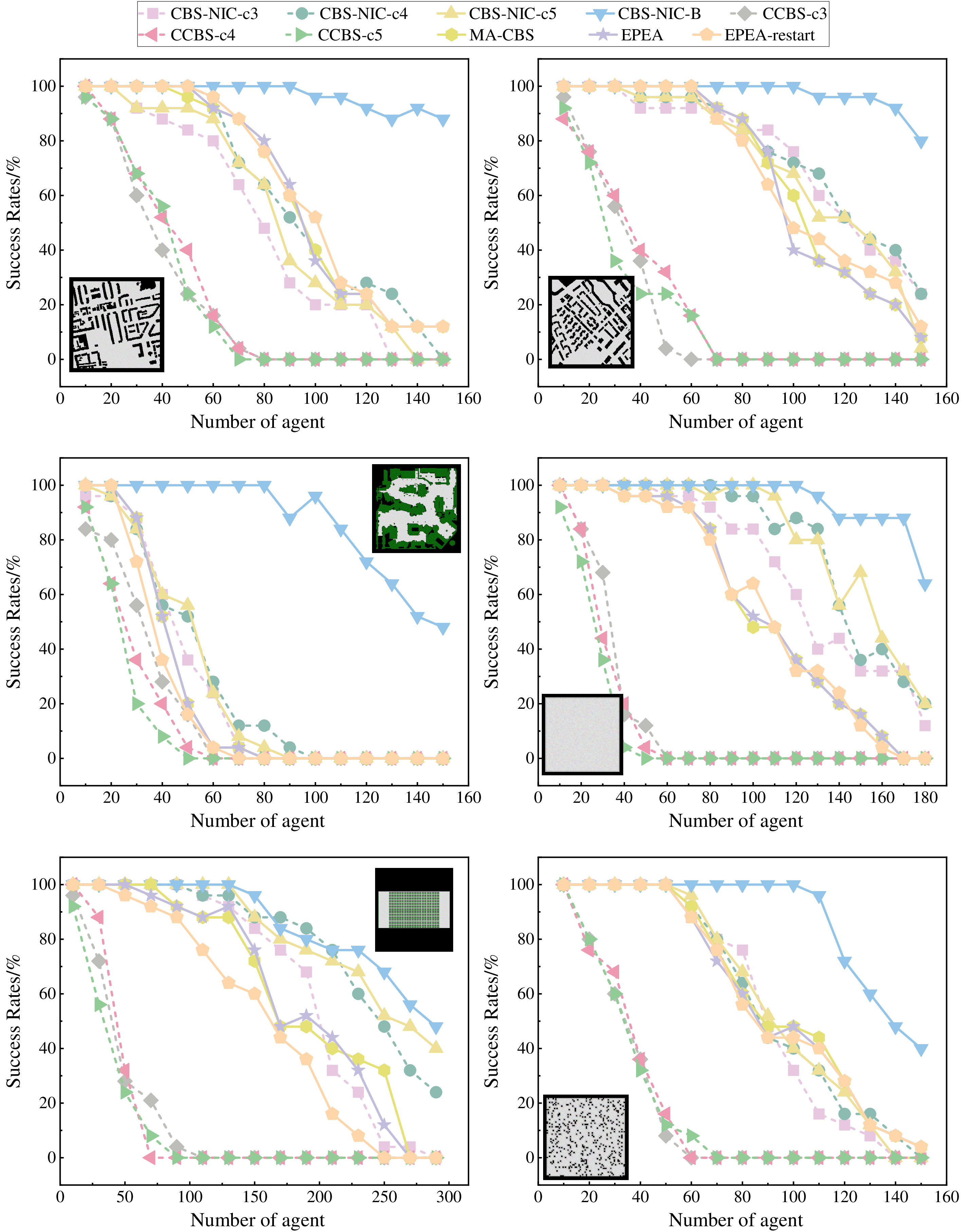}
		\caption{The success rate for different algorithms and their variants on gridmaps.}
		\label{fig5}
	\end{figure}

	\subsubsection{Success rate results}
	We evaluated the success rates of various algorithms within a 30-second time limit, considering different numbers of agents across various graphs.
	Fig. \ref{fig5} illustrates the success rate on gridmaps. 
	Across all tested graphs, CBS-NIC-B consistently achieves a higher success rate than other baseline algorithms. This improvement is attributed to the proposed BOGD strategy for graph discretization, which enables the selection of a more effective discretization parameter $s$, thereby enhancing solution efficiency. In contrast, CCBS suffers from low solution efficiency due to its unbounded continuous state space, limiting its scalability to scenarios with only $10–-30$ agents. Our proposed method, however, successfully scales to over 100 agents while maintaining a 100\% success rate.
	The CBS-NIC variant without the BOGD strategy demonstrates comparable performance to existing methods such as MA-CBS and EPEA. These baselines exhibit limitations: EPEA's partial expansion strategy offers limited benefits in densely constrained environments and may incur overhead that diminishes its scalability; MA-CBS introduces significant computational overhead due to the exponential growth of the joint search space when agents are merged into meta-agents. These factors collectively limit their effectiveness in large-scale or highly constrained scenarios.
	In the warehouse map, CBS-NIC-B shows limited improvement compared to other maps. This is because the many narrow aisles cause numerous neighbor nodes to be blocked by obstacles, reducing the effectiveness of the BOGD strategy.

	\subsubsection{Solution cost results}
	Fig.~\ref{fig6} presents box plots comparing the makespan of different algorithms on various grid maps. Since MA-CBS, EPEA, and its variant EPEA-restart yield nearly identical makespan results, we use MA-CBS as a representative in the figure. Similarly, CBS-NIC and MA-CBS show comparable performance, as both aim to minimize makespan by finding conflict-free paths using an optimal search strategy. In contrast, CBS-NIC-B typically results in a higher makespan than CBS-NIC, primarily due to the discretization of the graph. CCBS operates in continuous space and accounts for agent size, which also leads to a larger makespan compared to discrete-space methods. Notably, on the warehouse map, CBS-NIC-B exhibits a significantly higher makespan than other algorithms. This is attributed to the structural characteristics of the map, specifically the dense distribution of aisles, which imposes stricter constraints on discretized path planning.
	Overall, although CBS-NIC-B incurs a slight sacrifice in makespan optimality, it significantly improves the success rate of solving MAPF$_Z$ instances. For instance, on the Berlin gridmap with 50 agents, an increase of 5.83\% in makespan leads to a substantial 68\% improvement in the success rate.

	\begin{figure}
		\centering
		\includegraphics[width=3.5in]{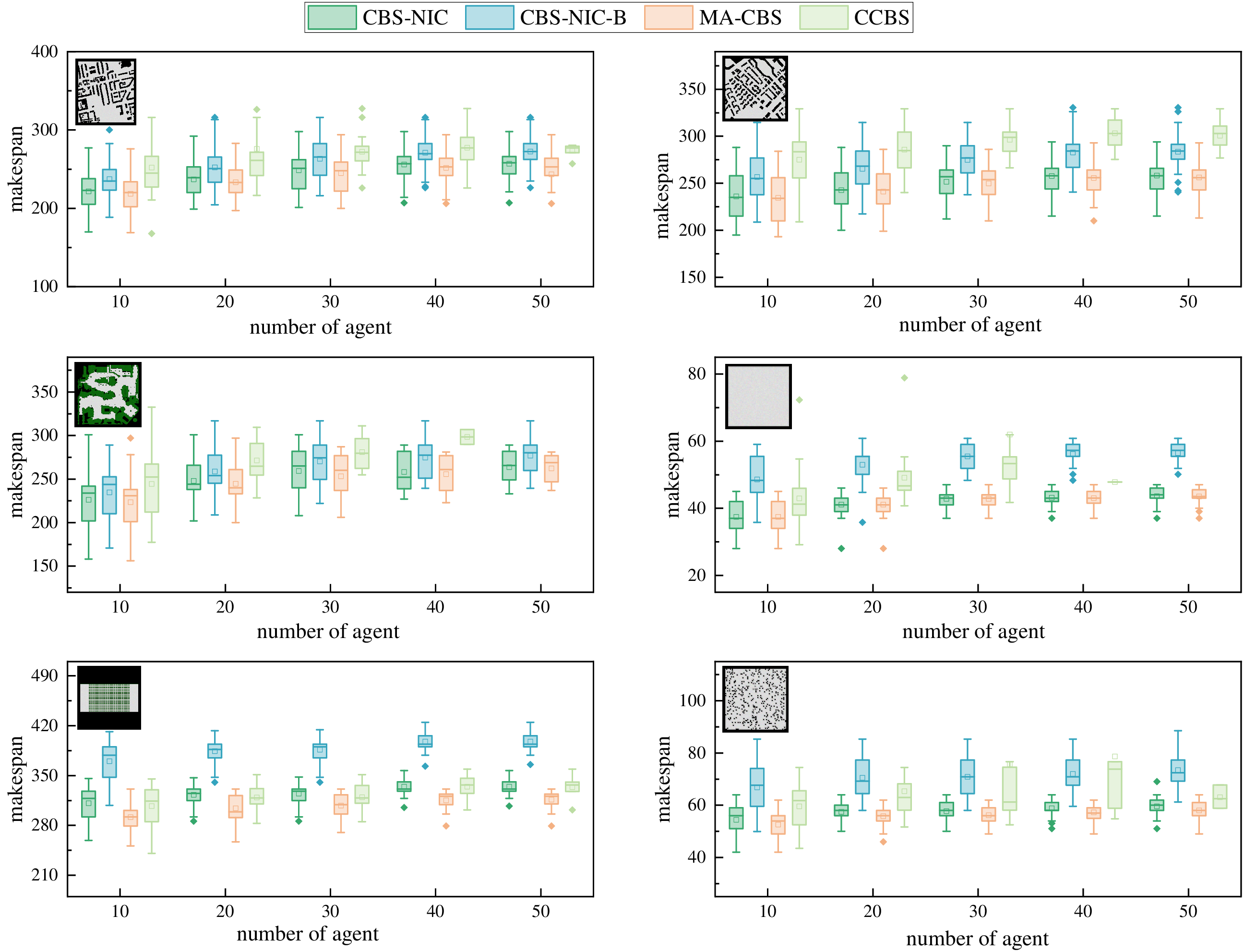}
		\caption{The makespan for different algorithms on gridmaps.}
		\label{fig6}
	\end{figure}
	
	\subsubsection{The runtime results}
	
	\begin{figure}
		\centering
		\includegraphics[width=3.5in]{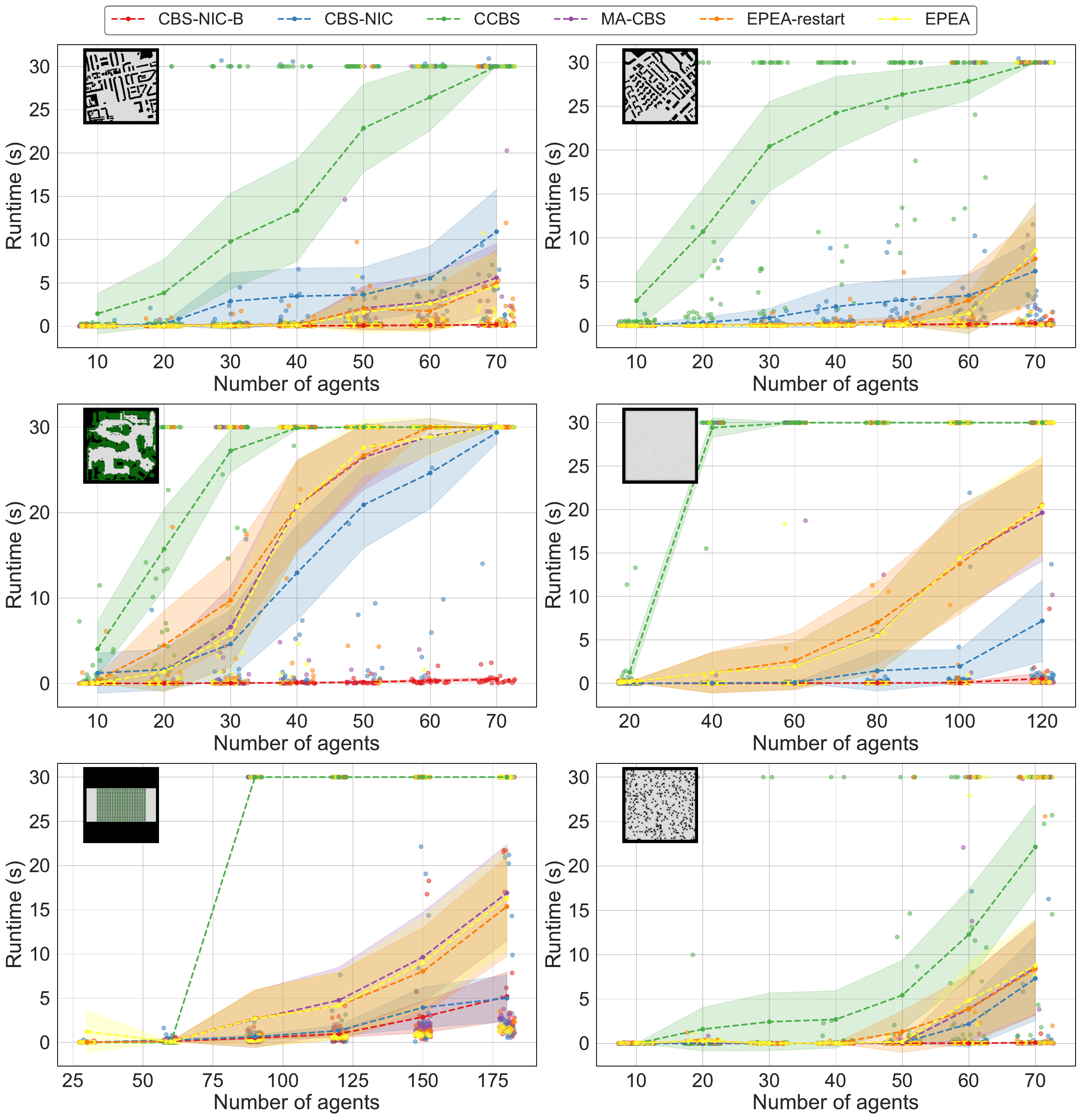}
		\caption{The runtime for different algorithms on gridmaps.}
		\label{fig7}
	\end{figure}
	Fig.~\ref{fig7} illustrates the runtime performance of different algorithms using both line graphs and scatter plots. The line graph represents the average runtime across multiple instances, while the scatter plot shows the distribution and variability of individual runtimes. A maximum runtime threshold of 30 seconds was set for all algorithms. When CCBS consistently reaches this upper limit, it indicates that the algorithm failed to find a conflict-free solution within the allotted time.
	From the figure, it is evident that the proposed CBS-NIC-B achieves the lowest runtime among all evaluated methods. Even as the number of agents exceeds 100, CBS-NIC-B is able to maintain execution times within approximately one second, demonstrating excellent scalability and efficiency. MA-CBS and EPEA follow in terms of performance, with moderate runtimes that grow with problem complexity. In contrast, CCBS incurs the highest runtime, which is primarily due to its operation in an unbounded, continuous state space and the additional overhead introduced by considering agent sizes and continuous-time collision checking. These results highlight the advantages of CBS-NIC-B in terms of both speed and robustness, especially in large-scale multi-agent scenarios.
	
	\begin{table}[htbp]
		\centering
		\renewcommand{\arraystretch}{1.0}
		\setlength{\tabcolsep}{4pt}
		\scriptsize
		\caption{Performance comparison of different methods across three map types and varying numbers of agents. }
		\begin{tabular}{c c c c c c}
			\toprule
			\multirow{2}{*}{\textbf{Map}} & \multirow{2}{*}{\textbf{\#Agents}} 
			& \multicolumn{4}{c}{\textbf{Succ. Rate (\%) / Runtime (s)}} \\
			\cmidrule(lr){3-6}
			& & \textbf{Baseline} & \textbf{CBS-NIC-B} & \textbf{CCBS+DP} & \textbf{CCBS+DP-1.3} \\
			\midrule
			
			\multirow{8}{*}{\textbf{den520d}} 
			& 10 & 100 / 0.698 & 100 / 0.004 & 72 / 3.165 & 72 / 0.011 \\
			& 20 & 100 / 1.588 & 100 / 0.007 & 4 / 19.885 & 20 / 0.023 \\
			& 30 & 72 / 7.865 & 100 / 0.025 & 0 / -- & 0 / -- \\
			& 40 & 28 / 7.749 & 100 / 0.087 & 0 / -- & 0 / -- \\
			& 50 & 4 / 10.125 & 100 / 0.088 & 0 / -- & 0 / -- \\
			& 70 & 0 / -- & 88 / 2.893 & 0 / -- & 0 / -- \\
			& 90 & 0 / -- & 52 / 2.494 & 0 / -- & 0 / -- \\
			& 110 & 0 / -- & 28 / 6.737 & 0 / -- & 0 / -- \\
			\midrule
			
			\multirow{7}{*}{\textbf{Berlin}} 
			& 10 & 100 / 0.011 & 100 / 0.008 & 100 / 0.036 & 100 / 0.053 \\
			& 20 & 100 / 0.115 & 100 / 0.013 & 92 / 5.450 & 88 / 2.494 \\
			& 30 & 100 / 3.338 & 100 / 0.045 & 32 / 8.052 & 56 / 0.664 \\
			& 40 & 76 / 6.954 & 100 / 0.117 & 0 / -- & 0 / 0.630 \\
			& 50 & 36 / 14.678 & 100 / 0.316 & 0 / -- & 0 / 0.238 \\
			& 70 & 0 / -- & 96 / 2.104 & 0 / -- & 0 / -- \\
			& 90 & 0 / -- & 52 / 7.375 & 0 / -- & 0 / -- \\
			\midrule
			
			\multirow{8}{*}{\textbf{random}} 
			& 10 & 100 / 0.002 & 100 / 0.002 & 92 / 0.025 & 84 / 0.020 \\
			& 20 & 100 / 0.006 & 100 / 0.003 & 68 / 2.266 & 56 / 0.579 \\
			& 30 & 100 / 0.054 & 100 / 0.006 & 12 / 5.760 & 16 / 3.911 \\
			& 40 & 100 / 0.357 & 100 / 0.014 & 0 / -- & 4 / 1.545 \\
			& 70 & 64 / 2.699 & 100 / 0.064 & 0 / -- & 0 / -- \\
			& 100 & 0 / 1.097 & 100 / 0.097 & 0 / -- & 0 / -- \\
			& 130 & 0 / 1.321 & 88 / 1.321 & 0 / -- & 0 / -- \\
			& 170 & 0 / 3.328 & 36 / 3.328 & 0 / -- & 0 / -- \\
			
			\bottomrule
		\end{tabular}
		
		{\scriptsize Each cell reports the success rate (\%) and average runtime (s), computed over successful instances only. “--” indicates failure in all cases.}
		\label{tab1}
	\end{table}
	
	\subsubsection{Performance on roadmaps}
	The roadmap evaluation results are presented in Table~\ref{tab1}, which reports both the success rate and runtime performance. Since other baseline algorithms cannot be directly applied to roadmap-based planning, we compare our approach with two variants of CCBS: CCBS with DS and PC (denoted as CCBS+DP), and a suboptimal version with a suboptimality bound of $w=1.3$ (denoted as CCBS+DP-1.3).
	From the results, it is evident that our proposed CBS-NIC and its variant CBS-NIC-B exhibit significantly better scalability compared to CCBS. On the random roadmap, CBS-NIC-B maintains a 100\% success rate even when the number of agents reaches 100, whereas CCBS fails to produce any valid solution beyond 40 agents. Similarly, the runtime of CCBS is several orders of magnitude higher than that of our methods. 
	The suboptimal version of CCBS (CCBS+DP-1.3) achieves improved performance over its optimal counterpart, which can be attributed to its trade-off between solution optimality and computational efficiency. By relaxing the optimality requirement, it gains in runtime and success rate at the cost of solution quality.
	These findings confirm that CBS-NIC-B is highly scalable and efficient even in roadmap-based planning scenarios.
	
	\subsubsection{Ablation Study on BOGD}
	
	To evaluate the effectiveness of BOGD, we conduct an ablation study by comparing it with a Baseline graph construction strategy, as summarized in Table~\ref{tab:graph_design}. Overall, incorporating BOGD consistently improves robustness and computational efficiency across different problem scales.
	From a success rate perspective, CBS-NIC-B maintains near-perfect performance in most settings and remains effective as the number of agents increases, whereas the Baseline degrades significantly and often fails in high-density scenarios. 
	In terms of runtime, CBS-NIC-B consistently achieves significant speedups, often reducing computation time by one to two orders of magnitude, especially in large-scale instances. This suggests that BOGD significantly enhances the solvability of the problem in high-density scenarios.
	
	\begin{table}
		\centering
		\footnotesize
		\setlength{\tabcolsep}{3pt}
		\renewcommand{\arraystretch}{0.6}
		
		\caption{Ablation study of BOGD: comparison between baseline and CBS-NIC-B}
		\label{tab:graph_design}
		
		\begin{tabular}{l c c c c c}
			\toprule
			\multirow{2}{*}{\textbf{Map}} & \multirow{2}{*}{\textbf{\#Agents}} & \multicolumn{2}{c}{\textbf{Baseline}} & \multicolumn{2}{c}{\textbf{CBS-NIC-B}} \\
			\cmidrule(r){3-4} \cmidrule(l){5-6}
			& & \textbf{Succ.(\%)} & \textbf{Runtime(s)} & \textbf{Succ.(\%)} & \textbf{Runtime(s)} \\
			\midrule
			
			\multirow{5}{*}{berlin}
			& 20  & 100 & 0.242 & 100 & 0.011 \\
			& 50  & 92  & 3.62  & 100 & 0.07  \\
			& 80  & 64  & 5.5   & 100 & 0.21  \\
			& 110 & 20  & 5.06  & 96  & 1.18  \\
			& 140 & 0   & --    & 92  & 2.08  \\
			\cmidrule{1-6}
			
			\multirow{5}{*}{boston}
			& 20  & 100 & 0.36  & 100 & 0.011 \\
			& 50  & 96  & 1.78  & 100 & 0.12  \\
			& 80  & 84  & 5.01  & 100 & 0.32  \\
			& 110 & 52  & 9.6   & 96  & 0.74  \\
			& 140 & 32  & 23.08 & 92  & 2.32  \\
			\cmidrule{1-6}
			
			\multirow{5}{*}{den520d}
			& 20  & 96  & 1.17  & 100 & 0.014 \\
			& 30  & 84  & 3.05  & 100 & 0.054 \\
			& 40  & 60  & 6.57  & 100 & 0.084 \\
			& 50  & 56  & 9.44  & 100 & 0.16  \\
			& 60  & 24  & 13.46 & 100 & 0.33  \\
			\cmidrule{1-6}
			
			\multirow{5}{*}{empty}
			& 40  & 100 & 0.039  & 100 & 0.0063 \\
			& 80  & 96  & 0.25   & 100 & 0.046  \\
			& 120 & 80  & 1.48   & 100 & 0.57   \\
			& 150 & 68  & 3.74   & 88  & 1.83   \\
			& 180 & 20  & 11.04  & 64  & 1.91   \\
			\cmidrule{1-6}
			
			\multirow{5}{*}{random}
			& 20  & 100 & 0.008  & 100 & 0.003 \\
			& 50  & 100 & 0.094  & 100 & 0.017 \\
			& 80  & 68  & 0.81   & 100 & 0.11  \\
			& 110 & 32  & 1.18   & 96  & 1.62  \\
			& 140 & 0   & --     & 48  & 6.3   \\
			
			\bottomrule
		\end{tabular}
	\end{table}

	\section{CONCLUSIONS}\label{sec6}
	In this paper, we introduced MAPF$_Z$, a novel variant of the multi-agent pathfinding problem that accommodates non-unit integer costs, thereby providing a more realistic modeling framework for scenarios with heterogeneous edge weights. To solve MAPF$_Z$ efficiently, we proposed CBS-NIC-B, which achieves notable performance gains by integrating time-interval-based conflict detection, refined constraint formulation, an enhanced SIPP algorithm, and the BOGD framework. 
	Experimental results demonstrate that CBS-NIC-B significantly improves success rates while maintaining competitive solution quality. 
	In future work, we plan to extend our algorithm to more challenging scenarios, such as large-scale environments. Furthermore, we aim to incorporate spatio-temporal planning considerations to better reflect real-time execution constraints and enhance the practical applicability of our approach.

	\section*{ACKNOWLEDGMENT}
	This work was supported by the National Natural Science Foundation of China under Grant No. 62373139,  Young Talents sponsorship Project of Hunan under Grant No.2023RC3121, Hunan University-China Mobile Communications Group Co.,Ltd. Industrial Intelligent Joint Institute, and China Mobile Limited.


	\bibliography{document.bib} 
	\bibliographystyle{IEEEtran} 
	
	\vspace{-16 mm}
	\begin{IEEEbiography}[{\includegraphics[width=1in,height=1.2in,clip,keepaspectratio]{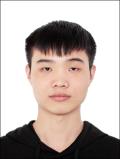}}]{Hongkai Fan}  received the B.S. degree in information engineering from the Hunan University of Technology, Zhuzhou, China, in 2021. He is currently pursuing the Ph.D. degree in electronic and information engineering with the College of Electrical and Information Engineering, Hunan University,Changsha, China. His research interests include the multi-agent systems and industrial manufacturing.
	\end{IEEEbiography}
	
	\vspace{-5mm}
	
	\begin{IEEEbiography}[{\includegraphics[width=1in,height=1.2in,clip,keepaspectratio]{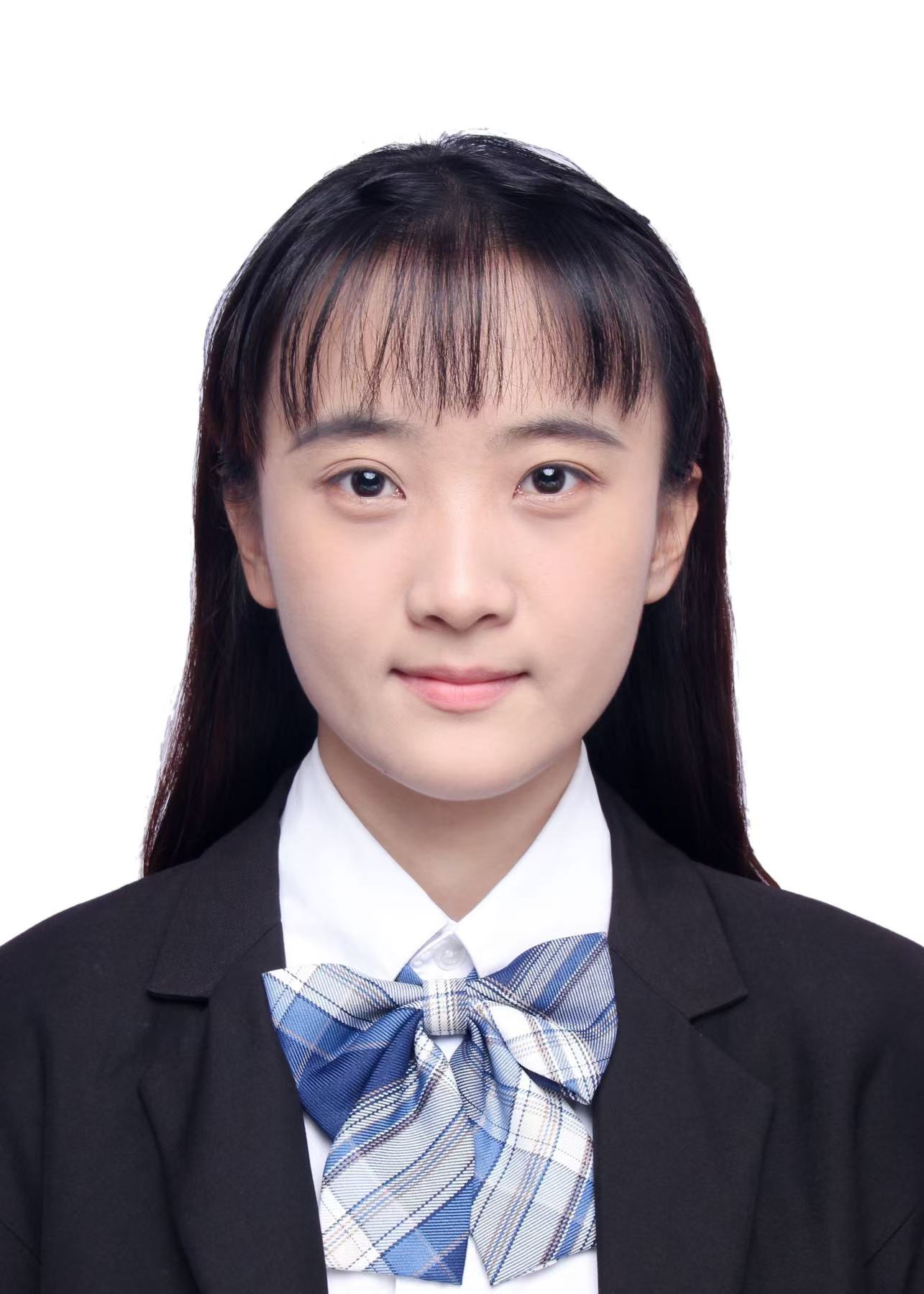}}]{Qinjing Xie} received her B.S. degree from Xiangtan University, Xiangtan, China, in 2023. She is currently pursuing her M.S. degree in Electronic and Information Engineering at the College of Electrical and Information Engineering, Hunan University, Changsha, China. Her research interests include multi-robot path graph optimization and layout optimization.
	\end{IEEEbiography}
	
	\vspace{-12mm}
	
	\begin{IEEEbiography}[{\includegraphics[width=1in,height=1.2in,clip,keepaspectratio]{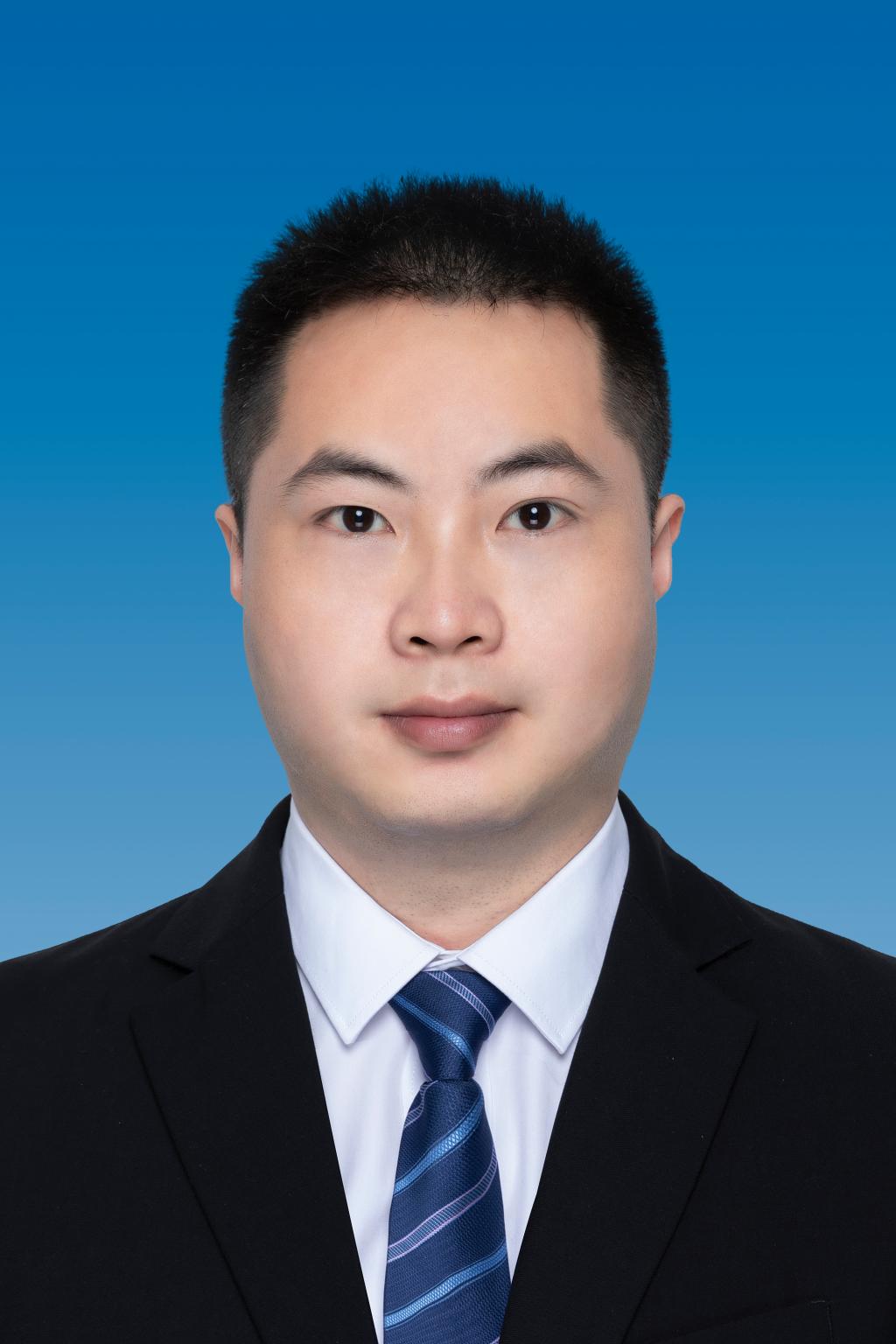}}]{Bo Ouyang} received the B.Eng. degree in electronic and information engineering from Zhejiang University China, in 2009 and the Ph.D. degree in circuits and systems from Zhejiang University, China in 2014. He is now an Associate Professor with the School of Electrical and Information Engineering, Hunan University, Changsha, China. His research interest includes the multi-agent system and complex networks.
	\end{IEEEbiography}
	
	\vspace{-12mm}
	
	\begin{IEEEbiography}[{\includegraphics[width=1in,height=1.2in,clip,keepaspectratio]{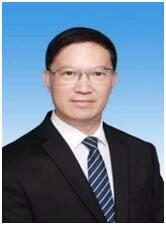}}]{Yaonan Wang} received the B.S. degree in computer engineering from East China Science and Technology University, Fuzhou, China, in 1981, and the M.S. and Ph.D. degrees in electrical
		engineering from Hunan University, Changsha, China, in 1990 and 1994, respectively.
		Since 1995, he has been a Professor with Hunan University. His research interests include
		robotics, intelligent perception and control, and computer vision for industrial applications.
		Prof. Wang is an Academician of the Chinese Academy of Engineering.
	\end{IEEEbiography}
	
	\vspace{-12mm}
	
	\begin{IEEEbiography}[{\includegraphics[width=1in,height=1.2in,clip,keepaspectratio]{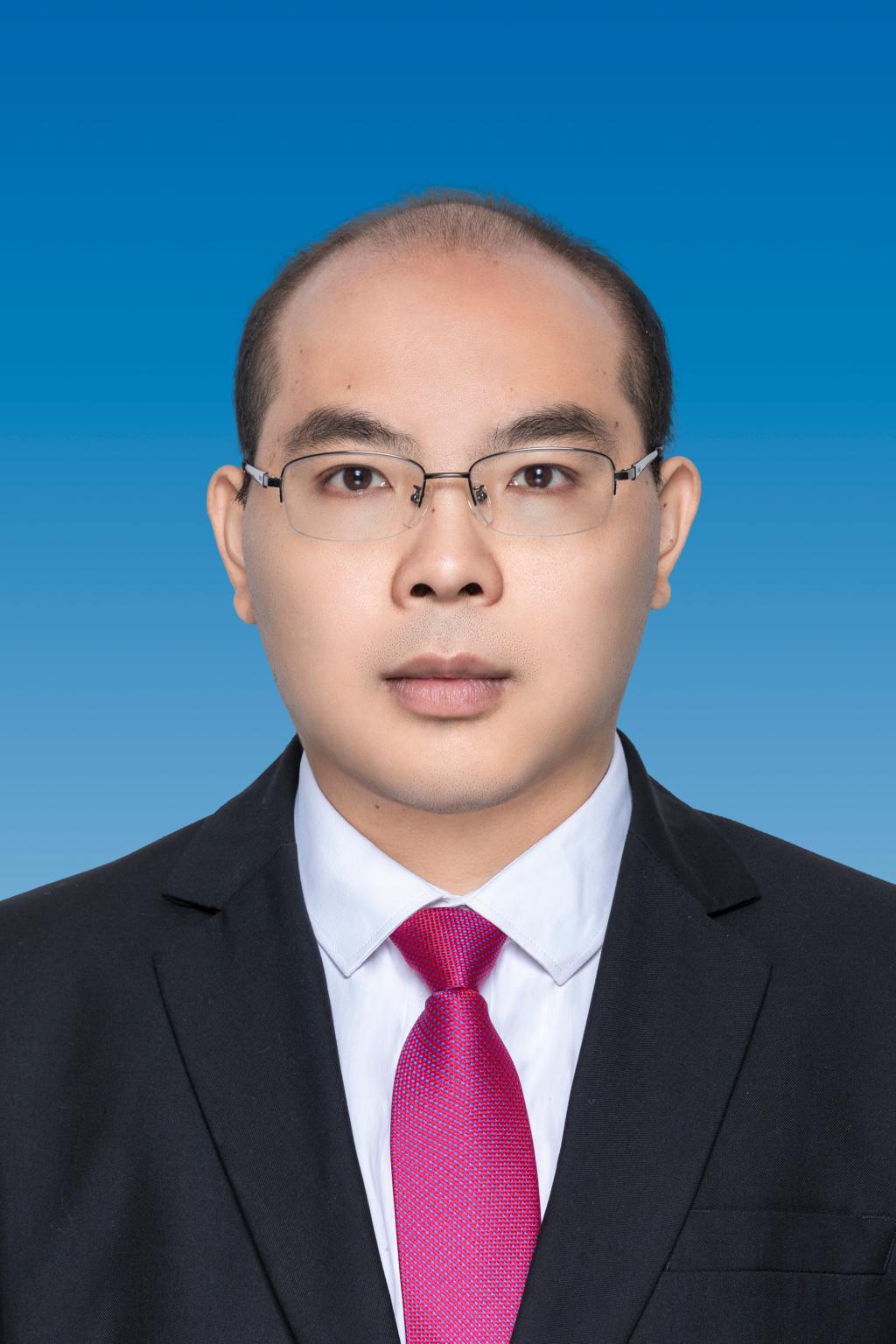}}]{Zhi Yan} received the B. Eng. degree in mechanical mechanical engineering and automation and the Ph.D. degree in communication and information system from Beijing University of Posts and Telecommunications (BUPT), Beijing, China, in 2007 and 2012, respectively. From August 2012 to March 2014, he was a researcher with the Network Technology Research Center, China Unicom Research Institute. He is currently an Associate Professor with the School of Electrical and Information Engineering, Hunan University, Changsha, China. His current research interests are cooperative communication, cellular network analysis and modeling.
	\end{IEEEbiography}
	
	\vspace{-12mm}
	
	\begin{IEEEbiography}[{\includegraphics[width=1in,height=1.2in,clip,keepaspectratio]{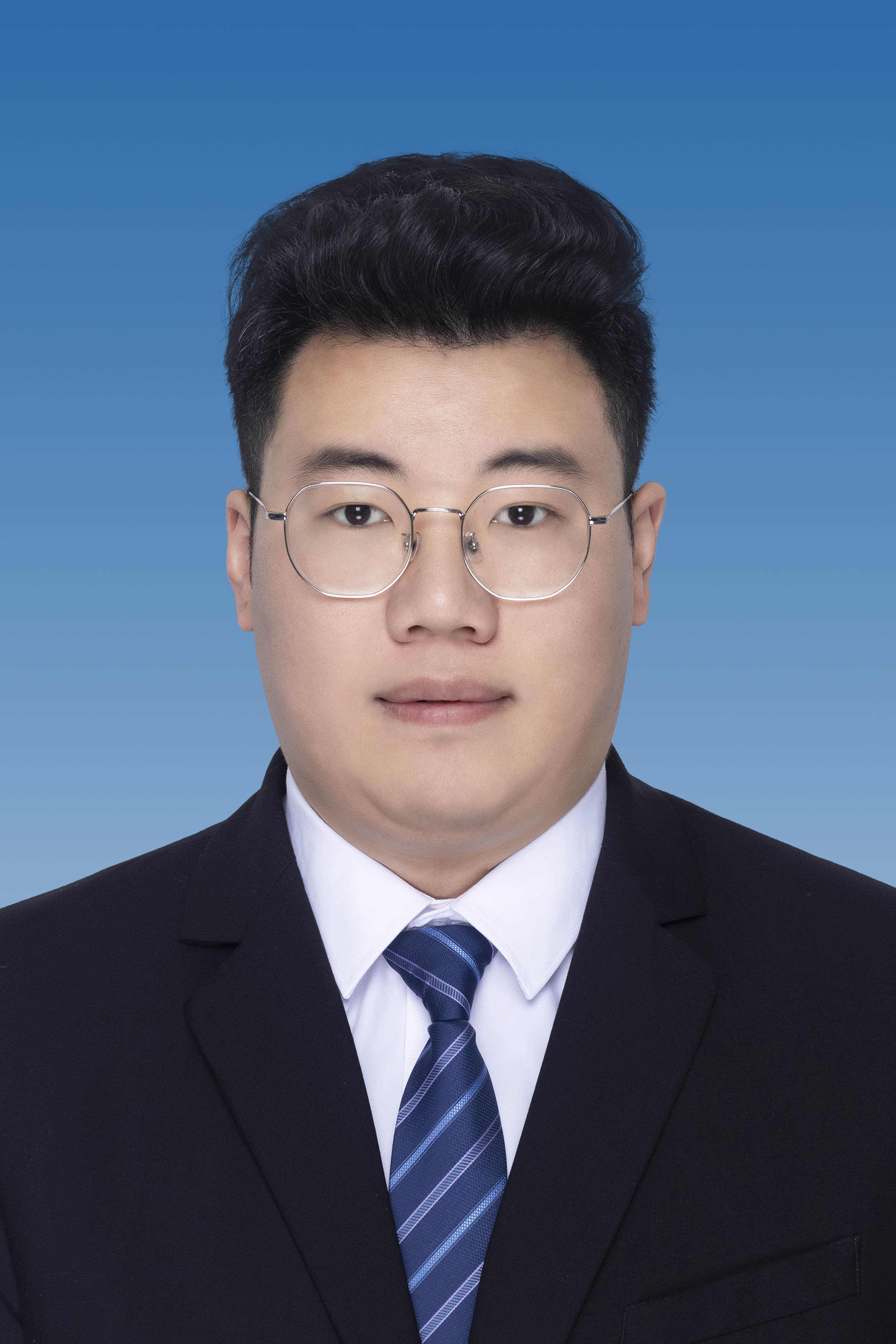}}]{Jiawen He} received the B.Eng. degree in mechanical design, manufacturing, and automation from Nanjing Agricultural University China, in 2019 and the M.S. degree in mechanical engineering from Central South University, Changsha, China, in 2022. He is currently working toward the Ph.D. degree in control science and engineering with Hunan University, Changsha, China. His current research interests include robotic control and industrial manufacturing.
	\end{IEEEbiography}
	
	\vspace{-12mm}
	
	\begin{IEEEbiography}[{\includegraphics[width=1in,height=1.2in,clip,keepaspectratio]{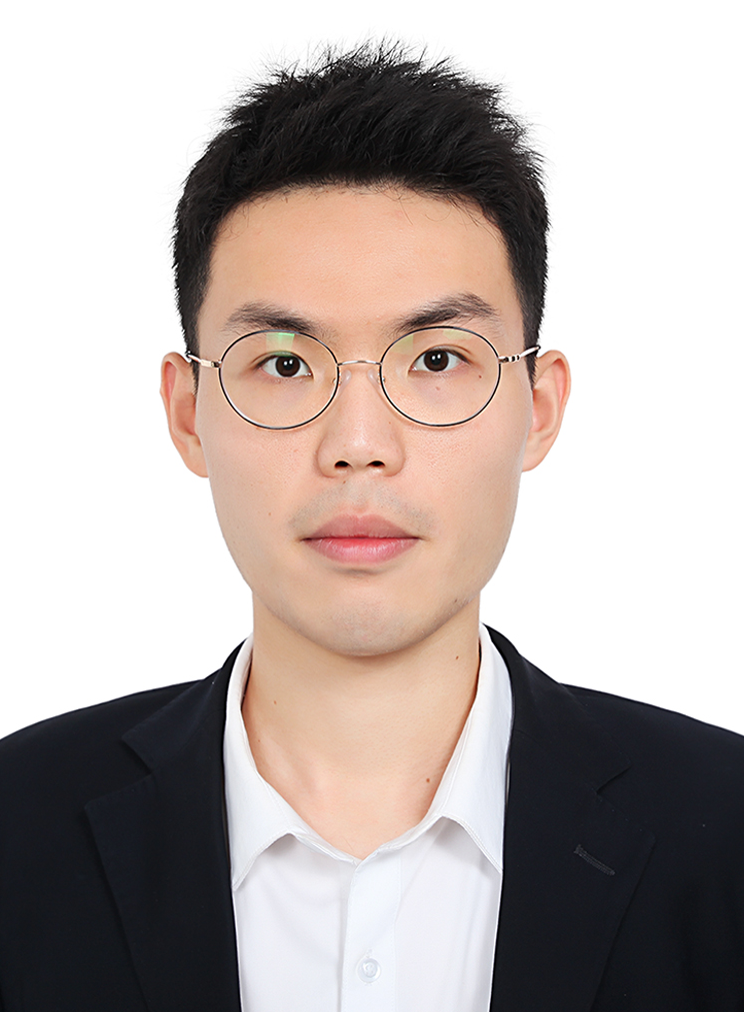}}]{ Zheng Fang} received the B.S. degree from Central South university, Changsha, China in 2018, and the Ph.D. degrees in physical electronics from Peking University, Beijing, China in 2023. Since 2023, he has been a researcher in China Mobile Group Hunan Company Limited, Changsha, China. His research interests include robotics and electronics for industrial applications.
	\end{IEEEbiography}

\end{document}